\documentclass[a4paper,11pt,DIV=11,
abstract=on
]{scrartcl}

\usepackage[utf8]{inputenc}
\usepackage[T1]{fontenc}
\usepackage[english]{babel}
\usepackage{amsmath,amssymb,amsfonts,siunitx,commath}
\usepackage{amsthm}
\usepackage{tikz,url}
\usepackage{geometry}
\usepackage{mathtools}
\mathtoolsset{showonlyrefs}
\usepackage{subcaption}
\usepackage{algorithm}
\usepackage{algpseudocode}
\usepackage{hyperref}
\usepackage{enumerate}
\usepackage{listings}
\usepackage{multirow}

\title{Title}
\author{Johannes Hertrich}
\date{}

\DeclareMathOperator*{\argmax}{argmax}
\DeclareMathOperator*{\argmin}{argmin}

\newcommand{\R}{\mathbb{R}}
\newcommand{\E}{\mathbb{E}}

\newcommand{\N}{\mathbb{N}}

\newcommand\prox{\mathrm{prox}}

\newcommand\trace{\mathrm{tr}}

\newcommand\SPD{\mathrm{SPD}}
\newcommand\St{\mathrm{St}}

\newcommand{\tT}{\mathrm{T}}

\newcommand{\n}{\phantom{0}}

\theoremstyle{plain}
\newtheorem{lemma}{Lemma}[section]
\newtheorem{theorem}[lemma]{Theorem}
\newtheorem{corollary}[lemma]{Corollary}
\newtheorem{proposition}[lemma]{Proposition}

\theoremstyle{definition}
\newtheorem{definition}[lemma]{Definition}

\newtheorem{remark}[lemma]{Remark}

\begin{document}

\title{
PCA Reduced Gaussian Mixture Models 
with Applications in Superresolution
}
\author{
Johannes Hertrich\footnotemark[1]
\and
Dang Phuong Lan Nguyen \footnotemark[2] \footnotemark[3]
\and
Jean-Francois Aujol \footnotemark[2]
\and
Dominique Bernard \footnotemark[4]
\and
Yannick Berthoumieu \footnotemark[3]
\and
Abdellatif Saadaldin \footnotemark[4]
\and
Gabriele Steidl\footnotemark[1]
}
\date{\today}

\maketitle

\footnotetext[1]{
TU Berlin,
Stra{\ss}e des 17. Juni 136, 
D-10587 Berlin, Germany,
\{j.hertrich, steidl\}@math.tu-berlin.de.
} 

\footnotetext[2]{
Univ. Bordeaux, Bordeaux INP, CNRS, IMB, UMR 5251, F-33400 Talence, France,
\{lan.nguyen,jean-francois.aujol\}@math.u-bordeaux.fr
}

\footnotetext[3]{
Univ. Bordeaux, Bordeaux INP, CNRS, IMS, UMR 5218, F-33400 Talence, France,
yannick.berthoumieu@ims-bordeaux.fr}

\footnotetext[4]{
CNRS, Univ. Bordeaux, Bordeaux INP, ICMCB, UMR 5026, F-33600 Pessac, France,
\{dominique.bernard,abdellatif.saadaldin\}@icmcb.cnrs.fr}

\begin{abstract}
Despite the rapid development of computational hardware, 
the treatment of large and high dimensional data sets is still a challenging problem.
This paper provides a twofold contribution to the topic. 
First, we propose a Gaussian Mixture Model in conjunction with a reduction of the
dimensionality of the data in each component of the model by 
principal component analysis, called PCA-GMM. 
To learn the (low dimensional) parameters of the
mixture model we propose an EM algorithm whose M-step requires
the solution of constrained optimization problems. 
Fortunately, these constrained problems do not depend on the usually large number of samples
and can be solved efficiently by 
an (inertial) proximal alternating linearized minimization algorithm.
Second, we apply our PCA-GMM for the superresolution of 2D and 3D material images
based on the approach of Sandeep and Jacob.
Numerical results confirm the moderate influence of the dimensionality reduction 
on the overall superresolution result.
\end{abstract}

\section{Introduction} \label{sec:intro}
The motivation for this work was superresolution of 3D material images taken within the 
project ITN MUMMERING (\texttt{https://www.mummering.eu}). Superresolution aims in
reconstructing high resolution images from low resolution ones.
Here a common assumption is that the low resolution image is generated by $y=Ax+\epsilon$, where $\epsilon$ is some noise and $A$ is a possible unknown superresolution operator.
Since this is an ill-posed inverse problem, methods addressing this task usually incorporate some prior information.
One approach for solving the problem is based on Gaussian Mixture Models (GMMs).
Usually, the GMM approximates the distribution of the patches of natural images 
and its parameters are learned from some given data, see also \cite{DSD2011}.

In literature, there exist several approaches to tackle superresolution by GMMs.
Zoran and Weiss \cite{ZW2011} proposed to use the negative log-likelihood function of a GMM as regularizer of the inverse problem by estimating the high resolution image $x_H$ given the low resolution one $x_L$ by solving
\begin{align}
\argmin_{x} \|Ax_H-x_L\|^2-\lambda \sum_{i\in I}\log p(x_{H,i}),\label{eq_MAP}
\end{align}
where $p$ is the probability density function of the GMM and $(x_{H,i})_{i\in I}$ are the patches in $x_H$. 
For a special choice of $\lambda$, the solution of this problem can be interpreted as 
the maximum a postiori (MAP) estimator of $x_H$ under the prior assumption 
that the distribution of the patches in $x_H$ are given by the GMM.
This method is called expected patch log likelihood (EPLL) and it can be applied  for several inverse problems.
Various accelerations and an efficient implementation of EPLL were developed in \cite{PDDN2019}.
However, EPLL requires that the operator $A$ is known, which is usually not the case for the superresolution task.
Therefore, we prefer the alternative approach of Sandeep and Jacob \cite{SJ2016}, 
which does not require any knowledge about the operator $A$.
While for EPLL the GMM describes only the distribution of the patches from the high resolution image, here the idea is to use a joint GMM, which describes the distribution of pairs of high and low resolution patches.
Having learned a joint GMM, each high resolution patch is estimated separately from
the low resolution patch and the joint GMM using the minimal mean squared error estimator.
 For a more detailed description of the method proposed by Sandeep and Jacob \cite{SJ2016}, we refer to Section \ref{sec:super}.

However, any of these models requires the estimation of the parameters of a GMM using the patches of the images as data points.
For this, the maximum likelihood (ML) estimator is used, which corresponds to minimizing the negative log likelihood function.
The standard method, to find the ML estimator is the expectation maximization (EM) algorithm \cite{Byrne2017,DLR1977}.
Unfortunately, the EM algorithm for GMMs becomes very slow, as the number of data points becomes large and high dimensional, which is the case for our superresolution task, particularly if we have to deal with 3D images.\\
To overcome these performance issues, we reduce the dimension of the data points.
The standard method for dimensionality reduction is the principal component analysis (PCA) \cite{Pearson1901}. 
The main assumption of the PCA is that the high dimensional data points are approximately located in a lower dimensional affine subspace.
In this paper, we combine the GMM with a PCA 
by adding the  minimization term of the PCA and the negative log likelihood function of the GMM on the dimensionality reduced data points. 
We rewrite this minimization problem again as the negative log likelihood function of a Gaussian mixture model which has additional constraints on the parameters. We call this model PCA-GMM.
This representation allows in particular, to use a different PCA for each component of the GMM.
We derive an EM algorithm with a special M-step for finding a minimizer of our objective function.
The M-step requires solutions of maximization problems with contraints on the Stiefel manifold.
Fortunately, these problems do no longer depend on the large number of sampling points
and they can be efficiently solved by the (inertial) proximal alternating linearized minimization algorithm (PALM)
for which some convergence results can be ensured.
The idea to couple parameter learning with dimension reduction is not new.
So the authors of \cite{BGS2006,HBD2018} propose directly a GMM with constraint covariance matrices. 
It is based on an extension of the PCA, which was proposed in \cite{TB1999} 
to replace the affine space in the PCA by the union of finitely many affine spaces
using a mixture model of probabilistic PCAs.
The relation to our approach is analyzed in a remark in Section \ref{sec:model}.                                                              
Using our new PCA-GMM model within the superresolution model of Sandeep and Jacob \cite{SJ2016}, 
we provide numerical examples of 2D and 3D material images. For making the examples reproducible, we provide the code online\footnote{\url{https://github.com/johertrich/PCA_GMMs}}.

The paper is organized as follows: 
in Section \ref{sec:prelim} we briefly review the two main ingredients for our model, namely GMMs and PCA.
We derive our PCA-GMM model in Section \ref{sec:model}.
In Section \ref{sec:alg}, an EM algorithm with a special constrained optimization task in the M-step is proposed
for minimizing the objective function. The solution of the constrained minimization problem via (inertial) PALM
is investigated.
The superresolution method using the PCA-GMM model is described in Section \ref{sec:super}.
Finally, Section \ref{sec:num} shows numerical examples of superresolution based on our PCA-GMM model 
on 2D and 3D images. Conclusions are drawn in Section \ref{sec:conclusions}.
\section{Preliminaries} \label{sec:prelim}
In this section, we briefly revisit the two building blocks of our approach, 
namely Gaussian mixture models and principal component analysis.
We need the following notation.
By $\SPD(n) \subset \R^{n,n}$ we denote the cone of \emph{symmetric positive definite $n \times n$ matrices},
by $O(n)$ the group of orthogonal $n \times n$ matrices,
by 
$$
\St(d,n) \coloneqq \bigl\{ U \in \mathbb R ^{n,d} : U^\tT U = I_d\bigr\}, \quad n \ge d,
$$
the \emph{Stiefel manifold} and by
$\Delta_K \coloneqq \{\mathbf{\alpha} = (\alpha_k)_{k=1}^K \in \R_{\ge 0}^K: \sum_{k=1}^K \alpha_k = 1\}$
the \emph{probability simplex}. We write $1_n$ for the vector with all $n$ components equal to 1. Further, we denote by $\|\cdot\|_F$ the Frobenius norm.

\paragraph{Gaussian Mixture Models}
The (absolutely continuous) \emph{$d$-dimensional normal distribution} 
$\mathcal N(\mu,\Sigma)$ with mean $\mu\in \R^n$ 
and positive semi-definite covariance matrix 
$\Sigma\in \SPD(n)$ has the  density
	\begin{equation}\label{density_normal}
		f(x|\mu,\Sigma) = (2\pi)^{-\frac{n }{2}} \abs{\Sigma}^{-\frac{1}{2}} 
		\,\exp\left(-\frac{1}{2}(x-\mu)^\tT \Sigma^{-1}(x-\mu) \right).
	\end{equation}  
	Note, that not all multivariate normal distributions are absolutely continuous,
	in particular, the covariance matrix is not necessarily invertible.
    However, for the rest of the paper, we focus on normal distributions with
    positive definite covariance matrices, which are invertible.
A \emph{Gaussian mixture model} (GMM) is a random variable 
with probability density function
\begin{equation}\label{eq:density_mm}
p(x)=\sum_{k=1}^K \alpha_k f(x|\mu_k,\Sigma_k), \qquad \alpha\in\Delta_K.
\end{equation}
For samples $\mathcal X = \{x_1,...,x_N\}$, the maximum likelihood (ML) estimator of the parameters 
${\mathbf \alpha} = (\alpha_k)_{k=1}^K$,
${\mathbf \mu} = (\mu_k)_{k=1}^K$ and 
${\mathbf \Sigma} = (\Sigma_k)_{k=1}^K$ 
of a GMM can be found
by minimizing its \emph{negative log-likelihood function} 
\begin{align}\label{eq:logLike}
L(\mathbf{\alpha},\mathbf{\mu},\mathbf{\Sigma}|\mathcal X)
&=
-\sum_{i=1}^N\log \Big(\sum_{k=1}^K\alpha_k f(x_i|\mu_k,\Sigma_k)\Big)
\end{align}
for
${\mathbf \alpha} \in \triangle_K$, $\mu_k \in \R^n$, and $\Sigma_k \in \SPD(n)$, $k=1,\ldots,K$.

In the following, we use the notation $\vartheta \coloneqq (\mu, \Sigma)$ to address the parameters of a Gaussian distribution.
A standard minimization algorithm for finding the ML estimator of the parameters $\alpha_k$ and
$\vartheta_k = (\mu_k, \Sigma_k)$, $k=1,\ldots,K$ of a GMM
is the so-called \emph{EM algorithm} \cite{Byrne2017,DLR1977} detailed in Alg. \ref{alg_em_mm}.
 
\begin{algorithm}[!ht]
\caption{EM Algorithm for Mixture Models}\label{alg_em_mm}
\begin{algorithmic}
\State Input: $x=(x_1,...,x_N)\in\R^{n\times N}$, initial estimate $\vartheta^{(0)}$.
\For {$r=0,1,...$}
\State \textbf{E-Step:} For $k=1,...,K$ and $i=1,\ldots,N$ compute 
$$
\beta_{i,k}^{(r)}=\frac{\alpha_k^{(r)}f(x_i|\vartheta_k^{(r)})}{\sum_{j=1}^K\alpha_j^{(r)}f(x_i|\vartheta_j^{(r)})}
$$
\State \textbf{M-Step:} For $k=1,...,K$ compute
\begin{align}
\alpha_k^{(r+1)}&=\frac1N \sum_{i=1}^N \beta_{i,k}^{(r)},\\
\vartheta_k^{(r+1)}&=\argmax_{\vartheta_k}\Big\{\sum_{i=1}^{N} \beta_{i,k}^{(r)}\log(f(x_i|\vartheta_k))\Big\}.
\end{align}
\EndFor
\end{algorithmic}
\end{algorithm}

For Gaussian density functions \eqref{density_normal}, the iterates
$\vartheta_k^{(r+1)}$, $k=1,\ldots,K$,
i.e. the maximization in the M-Step of Alg. \ref{alg_em_mm}
can be simply computed by
\begin{align} \label{easy}
\mu_k^{(r+1)} 
&= 
\frac{\sum_{i=1}^N\beta_{ik}^{(r)}x_i}{\sum_{i=1}^N\beta_{ik}^{(r)}} = \frac{1}{N\alpha_k^{(r+1)}} \, m_k^{(r)},
\\ 
\Sigma_k^{(r+1)} 
&=
\frac{\sum_{i=1}^N \beta_{ik}^{(r)} (x_i-\mu_k^{(r+1)})(x_i-\mu_k^{(r+1)})^\tT}{\sum_{i=1}^N \beta_{ik}^{(r)}}
=\frac{1}{N\alpha_k^{(r+1)}} \, C_k^{(r)}-\mu_k^{(r+1)}(\mu_k^{(r+1)})^\tT,\label{easy2}
\end{align}
where
\begin{equation}
m_k^{(r)} = \sum_{i=1}^N \beta_{ik}^{(r)}x_i 
\quad \text{and}\quad
C_k^{(r)} = \sum_{i=1}^N \beta_{ik}^{(r)}x_ix_i^\tT.
\end{equation}

\paragraph{Principal Component Analysis} In many applications, the dimension of the data is huge
such that dimensionality reduction methods become necessary. The working horse for dimensionality reduction
is the principal component analysis (PCA).
Given data samples
$\mathcal X= \{x_1,...,x_N\}$ in $\mathbb R^n$,
the classical PCA finds the $d$-dimensional affine space
$\{U \,  t +   b:  t \in \R^d\}$, $1 \le d \ll n$ having smallest squared distance
from the samples 
by minimizing
\begin{equation} \label{PCA_1}
         P(U,b) 
        = 
         \sum_{i=1}^N \|(U U^\tT -  I_n )( x_i - b)\|^2                                
\end{equation}
for $b \in \R^n$ and $U \in \St(d,n)$. 
It is not hard to check that the affine subspace 
goes through the \emph{offset} (\emph{bias}) 
$
b = \bar x \coloneqq \tfrac{1}{N}(x_1 + \ldots + x_N)
$
so that we can reduce our attention to the minimization with respect to $U \in \St(d,n)$, i.e., consider
$$
P(U)        = 
         \sum_{i=1}^N \|(U U^\tT -  I_n )y_i)\|^2 , \quad y_i = x_i - \bar x.
$$
Note, that a minimizer can be derived explicitly as the matrix $\hat U$, whose columns are given by the eigenvectors corresponding to the $d$ largest eigenvalues of the empirical covariance matrix $\sum_{i=1}^Ny_iy_i^\tT$. 
This minimizer is not unique, since it holds $P(UV)=P(U)$ for any orthogonal matrix $V\in O(d)$.

\section{PCA-GMM Model} \label{sec:model}
In this section, we propose a GMM which incorporates a dimensionality reduction model via PCA.
More precisely, we want to consider Gaussian distributions only on smaller subspaces of the original data space. 

A first idea would be to couple the GMM and the PCA model in an additive way and to minimize
for data samples $\mathcal X= \{x_1,...,x_N\}$ in $\mathbb R^n$ the function
\begin{equation}\label{init}
F(U,\alpha,\vartheta) = L \big(\mathbf{\alpha},\mathbf{\vartheta}|{\mathcal X}_{\mathrm{low}} \big) + \frac{1}{2 \sigma^2} P\left(U \right), \quad \sigma >0
\end{equation}
for $U \in \St(d,n)$, ${\mathbf \alpha} \in \triangle_K$, $\mu_k \in \R^d$, 
and $\Sigma_k \in \SPD(d)$, $k=1,\ldots,K$, where
$$
{\mathcal X}_{\mathrm{low}} \coloneqq \{U^\tT y_1, \ldots, U^\tT y_N\}, \quad y_i = x_i -\bar x.
$$
It is important that the negative log-likelihood function $L$ acts with respect to $\vartheta$ 
only on the lower dimensional space $\mathbb R^d$.
The function $F$ can be rewritten as
\begin{align}
F(U,\alpha,\vartheta) &= -\sum_{i=1}^N \bigg( \log \Big(\sum_{k=1}^K\alpha_k f\big(U^\tT y_i |\vartheta_k \big)\Big)
-  \frac{1}{2 \sigma^2} \|(U U^\tT -  I_n )y_i)\|^2 \bigg)\\
&=
-\sum_{i=1}^N \bigg( \log \Big( \sum_{k=1}^K \alpha_k f \big(U^\tT y_i |\vartheta_k \big) \, 
\exp\big(  -\tfrac{1}{2 \sigma^2}\|(U U^\tT -  I_n )y_i) \|^2\big) \Big) \bigg). \label{reform_1}
\end{align}
However, knowing that the samples were taken from $K$ different Gaussian distributions it makes more sense 
to reduce the dimension according to the respective distribution. Based on the reformulation \eqref{reform_1} and
using the notation  $\mathbf{U} = (U_k)_{k=1}^K$ and $\mathbf{b} = (b_k)_{k=1}^K$,
we propose to minimize the following \textbf{PCA-GMM model}:
\begin{align} \label{gmm_pca_1}
 F(\mathbf{U},\mathbf{b},\mathbf{\alpha},\mathbf{\vartheta}) 
\quad \mathrm{subject \; to} \quad                 
&\alpha \in  \Delta_K, \, U_k \in \St(d,n),\Sigma_k \in \SPD(d),\;  k=1,\ldots,K,
\end{align}
where $b_k \in \R^n$, $\mu_k \in \R^d$ and 
\begin{align}\label{eq_constraint_MM}
 F(\mathbf{U},\mathbf{b},\mathbf{\alpha},\mathbf{\vartheta})
&\coloneqq 
-\sum_{i=1}^N \log
\bigg(\sum_{k=1}^K \alpha_k f(U_k^\tT y_{ik} |\vartheta_k) \, \exp \big(-\tfrac{1}{2\sigma^2}
\|(I_n-U_k U_k^\tT)y_{ik}\|^2 \big) 
\bigg),\\
& \quad y_{ik} \coloneqq x_i-b_k, \quad k=1,\ldots,K,\, i=1,\ldots,N.
\end{align}
Clearly, if $U_k = U$ and $b_k = \bar x$ for all $k=1,\ldots,K$, we get back to model \eqref{init}.

The next lemma shows that our PCA-GMM model can be rewritten as a GMM model
whose parameters incorporate those of the PCA.

\begin{lemma}\label{anders}
Let $\mu \in \R^d$, $\Sigma \in \SPD(d)$, $U \in \St(d,n)$, $b\in \R^n$ and
let $f$ be the Gaussian density function \eqref{density_normal}.
Then the following relation holds true: 
\begin{align}\label{eq_mult}
f\big(U^\tT (x-b)|\mu,\Sigma \big) \, \exp\big(-\tfrac{1}{2\sigma^2} \|(I_n-U U^\tT)(x-b)\|^2 \big)
&= 
(2\pi\sigma^2)^{\frac{n-d}{2}} f(x|\tilde \mu,\tilde \Sigma),
\end{align}
where 
\begin{align}
\tilde \Sigma &=
\big(\tfrac{1}{\sigma^2} (I_n-U U^\tT) + U \Sigma ^{-1} U^\tT \big)^{-1} \in \SPD(n),
\label{eq_sig_new}\\
\tilde \mu &= \tilde \Sigma U \Sigma^{-1} \mu + b \in \mathbb R^n.\label{eq_mu_new}
\end{align}
\end{lemma}

\begin{proof}
1. First of all, we verify that the matrices $\tilde \Sigma$ are well defined, i.e. 
that $\tfrac{1}{\sigma^2}(I_n-U U^\tT) + U (\Sigma)^{-1} U^\tT$ is invertible. 
Let $\tilde U\in\R^{n,(n-d)}$ such that $V\coloneqq(U|\tilde U)$ is an orthogonal matrix.
Then we obtain
\begin{align}
V^\tT \tilde \Sigma^{-1} V =
V^\tT(\tfrac{1}{\sigma^2}(I-UU^\tT)+U\Sigma^{-1}U^\tT)V
=\tfrac{1}{\sigma^2}(I-V^\tT U U^\tT V) + V^\tT U \Sigma^{-1}U^\tT V.
\end{align}
Since $(V^\tT U)^\tT = U^\tT V =(I_d|0)$, 
this is equal to
\begin{align}\label{well-def}
V^\tT \tilde \Sigma^{-1} V = \tfrac{1}{\sigma^2}\left(\begin{array}{c|c}0&0\\\hline0&I_{n-d}\end{array}\right)
+
\left(\begin{array}{c|c}\Sigma^{-1}&0\\\hline0&0\end{array}\right)
&=
\left(\begin{array}{c|c}\Sigma^{-1}&0\\\hline 0&\tfrac1{\sigma^2}I_{n-d}\end{array}\right)
\end{align}
and the last matrix is invertible. 

2. We have to show that
\begin{align*}
&(2\pi)^{-\frac{d}{2}}|\Sigma|^{-\frac12} 
\exp \Big( - \frac{1}{2\sigma^2} \|(I_n - U U^\tT) (x-b)\|^2 - \frac12 (U^\tT(x-b) -\mu)^\tT \Sigma^{-1}  (U^\tT(x-b) -\mu) \Big)\\
&=  (2\pi)^{-\frac{n}{2}} |\tilde \Sigma|^{-\frac12} \exp \Big( -\frac12(x - \tilde \mu)^\tT \tilde \Sigma^{-1} (x- \tilde \mu) \Big) \\
&=  (2\pi)^{-\frac{n}{2}} |\tilde \Sigma|^{-\frac12} \exp \Big( -\frac12 x^\tT  \tilde \Sigma^{-1} x +
\tilde \mu^\tT \tilde \Sigma^{-1} x - \frac12 \tilde \mu^\tT \tilde \Sigma^{-1} \tilde \mu \Big). 
\end{align*}
Straightforward calculation together with the observation that
$U^\tT \tilde \Sigma U = \Sigma$ and hence $U^\tT \tilde \Sigma^{-1} U = \Sigma^{-1}$ gives
\begin{align}
&\frac{1}{2\sigma^2} \|(I_n - U U^\tT) (x-b)\|^2 + \frac12 (U^\tT(x-b) -\mu)^\tT \Sigma^{-1}  (U^\tT(x-b) -\mu) \\
&=
\frac{1}{2} x^\tT \left(\tfrac{1}{\sigma^2}(I_n-U U^\tT) + U \Sigma^{-1} U^\tT \right) x
- \left(\tfrac{1}{\sigma^2} b^\tT (I_n - U U^\tT) + (\mu^\tT + b^\tT U) \Sigma^{-1} U^\tT \right) x\\
& \quad + \frac12 (U^\tT b + \mu)^\tT \Sigma^{-1} (U^\tT b + \mu) + \tfrac{1}{2\sigma^2} b^\tT (I_n-U U^\tT) b\\
&=
\frac12 x^\tT \tilde \Sigma^{-1} x - \tilde \mu^\tT \tilde \Sigma^{-1} x + \frac12 \tilde \mu^\tT \tilde \Sigma^{-1} \tilde \mu.
\end{align}
Finally, we see by \eqref{well-def} that $|\tilde \Sigma|^{-1} = \sigma^{-2(n-d)} |\Sigma|^{-1}$ .
\end{proof}

By Lemma \ref{anders}, we can rewrite our objective function $F$ in \eqref{eq_constraint_MM} 
with $\tilde \vartheta = (\tilde \mu, \tilde \Sigma)$ defined by \eqref{eq_sig_new} and \eqref{eq_mu_new} 
with corresponding indices as 
\begin{align}
F(\mathbf{U},\mathbf{b},\mathbf{\alpha}, \mathbf{\vartheta}) 
&= 
-\sum_{i=1}^N \log\Big(\sum_{k=1}^K \alpha_k f(x_i|\tilde \vartheta_k)\Big)
+ (n-d) \log \sqrt{2\pi \sigma^2}\\
&= L(\mathbf{\alpha},\tilde{\vartheta}|\mathcal X) + (n-d) \log \sqrt{2\pi \sigma^2}. \label{L*}
\end{align}
Up to the constant this is a negative log-likelihood function of a GMM. However,
when minimizing this function, we have to take the constraints \eqref{eq_sig_new} and \eqref{eq_mu_new} into account.
More precisely, our model in \eqref{gmm_pca_1} can be rewritten as \textbf{PCA-GMM model}:
\begin{align} \label{model}
\mathrm F(\mathbf{U},\mathbf{b},\mathbf{\alpha}, \mathbf{\vartheta}) 
\coloneqq L(\mathbf{\alpha},\tilde{\mathbf{\vartheta}}|\mathcal X)
\quad \mathrm{subject \; to} \quad U_k \in \St(d,n), \, \alpha \in \Delta_K,\,  \Sigma_k \in \SPD(d), 
\end{align}
where
\begin{align} \label{subst_1}
\tilde \Sigma_k &=
\big(\tfrac{1}{\sigma^2} (I_n-U_k U_k^\tT) + U_k \Sigma_k ^{-1} U_k^\tT \big)^{-1}, \; 
\tilde \mu_k = \tilde \Sigma_k U \Sigma_k^{-1} \mu_k + b_k, \quad k=1,\ldots,K.
\end{align}

The choice of $\mu_k$ and $b_k$ is redundant. This can be seen as follows, for any $\mu_k$ and $b_k$, define $\hat \mu_k=0$ and $\hat b_k=\tilde \mu_k$. Then, it holds
\begin{align}
\tilde \Sigma_k &=
\big(\tfrac{1}{\sigma^2} (I_n-U_k U_k^\tT) + U_k \Sigma_k ^{-1} U_k^\tT \big)^{-1}, \; 
\tilde \mu_k = \tilde \Sigma_k U \Sigma_k^{-1} \hat\mu_k + \hat b_k, \quad k=1,\ldots,K
\end{align}
such that $\mathrm{F}(\mathbf{U},\hat{\mathbf{b}},\mathbf{\alpha},\hat{\mathbf{\vartheta}})=\mathrm{F}(\mathbf{U},\mathbf{b},\mathbf{\alpha},\mathbf{\vartheta})$.
Consequently, in the M-step of the EM algorithm in Section \ref{sec:alg-em} we obtain that the update for the mean $\mu$ is given by $\mu_k=0$.

\begin{remark}[Different component dimensions]\label{rem_different}
So far the dimension $d$ is the same for all components $k=1,...,K$. 
But by some simple adjustments, the PCA-GMM model can also be rewritten with $U_k\in\St(d_k,n)$, $\mu_k\in\R^{d_k}$ and $\Sigma_k\in\SPD(d_k)$, where the $d_k$ are not necessarily equal for all $k$.
However, to keep the notations as simple as possible, we will restrict our analysis to the case that $d_k=d$ for $k=1,...,K$.
Nevertheless,  all the results of this paper can be derived analogously for other choices of $d_k$.
\end{remark}

\begin{remark}[Learning $\sigma$]\label{rem_learn_sigma}
The function $F$ in \eqref{init}, resp. \eqref{eq_constraint_MM}, \eqref{L*} is strictly decreasing in $\sigma$. 
Thus it does not make sense to minimize $F$ with respect to $\sigma$.

However, the function $\mathrm F = F-\tfrac{n-d}{2}\log(2\pi\sigma^2)$ in \eqref{model}
can be optimized with respect to $\sigma$. To keep the M-step of the EM algorithm simple,
we associate to each summand in the mixture model an own $\sigma_k$, $k=1,\ldots,K$ 
such that $\tilde \Sigma$ in \eqref{subst_1} becomes
\begin{align} \label{subst_2}
\tilde \Sigma_k &=
\big(\tfrac{1}{\sigma_k^2} (I_n-U_k U_k^\tT) + U_k \Sigma_k ^{-1} U_k^\tT \big)^{-1}.
\end{align}
In this case, we use the notation $\mathbf{\sigma}=(\sigma_k)_{k=1}^K$.
\end{remark}

\subsection*{Related Work}
There are several relations of the PCA-GMM model to other models proposed in the literature, in particular to mixtures of probabilistic PCAs (MPPCA) \cite{TB1999}, high dimensional data clustering (HDDC) \cite{BGS2006} and high-dimensional mixture models for unsupervised image denoising (HDMI) \cite{HBD2018}.
In the following, we shortly review these methods and comment on similarities and differences to the PCA-GMM model.\\

For understanding the relation to other models, we first need the following reformulation of the covariance matrices $\tilde \Sigma$ from the PCA-GMM model.
We have as in \eqref{well-def} for matrices $\tilde\Sigma$ of the form \eqref{eq_sig_new}
and an orthogonal matrix $V=(U|\tilde U)$ that
\begin{align}
V^\tT\tilde\Sigma V&
= \left(\begin{array}{c|c}\Sigma&0\\\hline 0&\sigma^2I_{n-d}\end{array}\right),
\end{align}
so that
\begin{align}
\bigg\{
&\tilde \Sigma = \big(\tfrac{1}{\sigma^2} (I_n-U U^\tT) + U \Sigma ^{-1} U^\tT \big)^{-1}: \,
 U \in \St(d,n), \, \Sigma \in \SPD(d) 
\bigg\}
\\
&=
\bigg\{ Q^\tT \left(\begin{array}{c|c}\mathrm{diag} (\lambda) &0\\\hline 0&\sigma^2I_{n-d}\end{array}\right)Q: \,
Q\in O(n), \, \lambda \in \mathbb R^d_{>0} \bigg\}. \label{subst}
\end{align}
As outlined in Remark \ref{rem_learn_sigma}, the $\sigma$ can either be fixed a priori, 
or optimized within the EM algorithm, as later outlined in Section \ref{sec:alg-em}, simultaneously with the other parameters.\\

In \cite{TB1999}, Tipping and Bishop propose mixture models of probabilistic PCAs (MPPCA), which are GMMs of the form
\begin{align}\label{dens_GMM}
p(x)=\sum_{k=1}^K\alpha_k f(x_i|\tilde \mu_k,\tilde \Sigma_k),
\end{align}
where 
$$
\tilde\Sigma_k=U_kU_k^\tT+\sigma_k^2I_n,\quad U_k\in\St(d_k,n).
$$
Here, the parameters $\sigma_k$  are optimized simultaneously with the $\alpha_k$ and $U_k$ via the EM algorithm.
Hence, skipping the index, instead of minimizing over \eqref{subst}, they minimize over sets of the form
\begin{align}\label{subst_MPPCA}
\bigg\{Q^\tT 
\left(\begin{array}{c|c}(1+ \sigma^2) I_d&0\\\hline 0&\sigma^2I_{n-d}\end{array}\right)
Q:
Q\in O(n)\bigg\}.
\end{align}
Since this form of the covariance matrices is very restrictive,  Bouveyron,  Girard and Schmid generalized MPPCA in \cite{BGS2006} to a model called high dimensional data clustering (HDDC). 
Again, they minimize a special GMM \eqref{dens_GMM}, but here the covariances are given by
$$
\tilde\Sigma_k=U_k\mathrm{diag}(\lambda_k)U_k^\tT+\sigma_k^2I_n,\quad U_k\in\St(d_k,n),\lambda_k\in\R_{>0}^{d_k}.
$$
As for the MPPCA, the parameters are optimized via the EM algorithm. 
For deriving it, it is important that the parameters $\sigma_k$ are not fixed a priori but are optimized within the EM algorithm. 
Skipping the index, instead of minimizing over \eqref{subst} or \eqref{subst_MPPCA}, this corresponds to a minimization over 
\begin{align}\label{subst_HDDC}
\bigg\{Q^\tT 
\left(\begin{array}{c|c}\mathrm{diag}(\lambda)+\sigma^2 I_d&0\\\hline 0&\sigma^2I_{n-d}\end{array}\right)
Q:
Q\in O(n), 
\lambda\in \mathbb R^d_{>0} \bigg\}.
\end{align}
In contrast to \eqref{subst}, where the diagonal values $\lambda$ are required to be strictly greater than $0$, the diagonal values $\lambda+\sigma^2$ in \eqref{subst_HDDC} are automatically strictly greater than $\sigma^2$.
Consequently, the PCA-GMM model is more general than HDDC.
Note that HDDC model contains the so-called mixture factor analysis \cite{MPB2003} as a special case. Here also the alternating expectation conditional maximization algorithm \cite{MV1997} is applicable \cite{ZP2008}, which is an improved version of the EM algorithm.
\\

Finally, Houdard, Bouveyron and Delon proposed in \cite{HBD2018} a model selection algorithm for the dimensions $d_k$.
For this, they propose a model called HDMI, where the only difference to HDDC is, that $\sigma$ is a priori fixed.
They derive as an intermediate step a corresponding EM algorithm in \cite[Proposition 2]{HBD2018}.
Unfortunately, the M-step only ensures that $\lambda >-\sigma^2  1_d$ and not $\lambda>0$, such that the calculations appear to be not fully correct.
However, the final model selection algorithm again ensures that $\lambda>0$ 
such that this seems not to be a problem in \cite{HBD2018}.

\section{Minimization Algorithm} \label{sec:alg}
We propose to minimize $\mathrm F$ in \eqref{model}
based on the EM algorithm, where we have to take the special structure
of $\tilde \mu_k \in \R^n$ and $\tilde \Sigma_k \in \SPD(n)$ in \eqref{subst_1}
into account to work indeed in the lower $d$-dimensional space.
This requires the solution of a special inner minimization problem within the
M-Step of the EM algorithm. 
We describe the EM algorithm for our PCA-GMM model in Subsection \ref{sec:alg-em}.
In particular, we will see that the M-Step of the algorithm requires the minimization
of functions $G_k(U,b)$, $k=1,\ldots,K$ of the same structure.
We prove that these functions have indeed a global minimizer.
In particular, these functions do not depend on the large number of input data $x_i$, $i=1,\ldots,N$.
Therefore it turns out that the E-step of the algorithm is the most time consuming one.
We propose to find at least a local minimizer of $G$ by the (inertial) Proximal alternating linearized minimization
(PALM) in Subsection \ref{sec:alg-palm} and provide convergence results.

\subsection{EM Algorithm for PCA-GMM} \label{sec:alg-em}
For our setting, we obtain a special EM algorithm described in Algorithm \ref{alg_em_gmm_pca}.
Note that E-Step of Algorithm \ref{alg_em_gmm_pca} requires only the mean 
and covariance matrix in $\vartheta_k^{(r)}$, $k=1,\ldots,K$ with respect to the smaller 
space $\R^d$.

\begin{algorithm}[!ht]
\caption{EM Algorithm for PCA reduced Mixture Models}\label{alg_em_gmm_pca}
\begin{algorithmic}
\State Input: $x=(x_1,...,x_N)\in\R^{n, N}$, 
initialization $\mathbf{U}^{(0)}$, $\mathbf{b}^{(0)}$, $\mathbf{\alpha}^{(0)}$, 
$\mathbf{\vartheta}^{(0)} = (\mathbf{\mu}^{(0)}, \mathbf{\Sigma}^{(0)})$.
\For {$r=0,1,...$}
\State \textbf{E-Step:} For $k=1,...,K$ and $i=1,\ldots,N$ compute 
\begin{align}
\beta_{i,k}^{(r)}
&=
\frac{\alpha_k^{(r)} f(x_i|\tilde \vartheta_k^{(r)}) }
{\sum_{j=1}^K \alpha_j^{(r)} f (x_i|\tilde \vartheta_j^{(r)} )}
\\
&=
\frac{ 
\tfrac{\alpha_k^{(r)}}{(\sigma_k^{(r)})^{n-d}}\exp \left(-\tfrac{1}{2 (\sigma_k^{(r)})^2} \| (I_n - U_k^{(r)} (U_k^{(r)})^\tT ) y_{i,k}^{(r)}\|^2 \right)
f\left( (U_k^{(r)})^\tT y_{i,k}^{(r)}| \vartheta_k^{(r)}\right)
}
{
\sum_{j=1}^K \tfrac{\alpha_j^{(r)}}{(\sigma_j^{(r)})^{n-d}}
\exp \left(-\tfrac{1}{2(\sigma_j^{(r)})^2} \|(I_n-U_j^{(r)}(U_j^{(r)})^\tT )
y_{i,k}^{(r)} \|^2 \right) f\left( (U_j^{(r)})^\tT
y_{i,k}^{(r)}| \vartheta_j^{(r)} \right),
}\\
y_{i,k}^{(r)} &= x_i-b_k^{(r)}.
\end{align}
\State \textbf{M-Step:} For $k=1,...,K$ compute
\begin{align}
\alpha_k^{(r+1)}
& = \frac1N \sum_{i=1}^N \beta_{i,k}^{(r)},\\
(U_k^{(r+1)},b_k^{(r+1)},\sigma_k^{(r+1)},\vartheta_k^{(r+1)})
&=
\argmax_{U,b,\mu,\Sigma}\sum_{i=1}^N\beta_{ik}^{(r)}
\log(f(x_i|\tilde \vartheta_k)) \label{to_solve}\\
&\mathrm{subject \; to} \quad U_k \in \St(d,n), \Sigma_k \in \SPD(d) \\
&\mathrm{with} \; \tilde \vartheta_k =(\tilde \mu_k, \, \tilde \Sigma_k) \; \mathrm{as\;  in} \;\eqref{subst_1}.
\end{align}
\EndFor
\end{algorithmic}
\end{algorithm}

A convergence analysis of the EM algorithm via Kullback-Leibler proximal point algorithms was 
given in \cite{CH2000, CH2008}, see also \cite{Laus2019} for a nice review.
The authors showed that the objective function decreases for the iterates of the algorithm.
Hence we obtain the following corollary.

\begin{corollary}\label{cor:EM}
For the iterates $\left(\mathbf U^{(r)}, \mathbf b^{(r)}, \alpha^{(r)},\mathbf \vartheta^{(r)}\right)_r$  
generated by Algorithm \ref{alg_em_gmm_pca} 
the objective function $\mathrm F$ is decreasing.
\end{corollary}

The interesting step is the second M-Step which requires again the maximization of a function.
Based on \eqref{easy} and \eqref{easy2} we can prove the following proposition.

\begin{proposition}\label{mstep}
Assume that $n+1$ of the points $x_i$, $i=1,...,N$ are affinely independent.

Further, let $f$  be the Gaussian density function \eqref{density_normal} and $\beta_i \in \R_{\ge 0}$, $i=1,\ldots,N$.
and let $\sigma^2$ be fixed.
\\
\textrm{i)} For fixed $\sigma^2$, a solution of 
\begin{align}\label{funct}
\argmax_{U,b,\mu,\Sigma}\sum_{i=1}^N \beta_{i}
\log(f(x_i|\tilde \vartheta))
\end{align}
with $\tilde \vartheta = (\tilde \mu,\tilde \Sigma)$ of the form \eqref{eq_mu_new} and  
\eqref{eq_sig_new}
is given by
\begin{align} \label{mu_sig}
\hat \mu = 0, 
&\quad 
\hat \Sigma =  \frac1{\alpha} \hat U^\tT  S   \hat U,\quad\text{and}\quad \hat b=\frac{1}{\alpha}m,
\end{align}
where
\begin{align}
m&=\sum_{i=1}^N \beta_i x_i,\quad 
C=\sum_{i=1}^N \beta_i x_i x_i^\tT,\quad 
\alpha = \sum_{i=1}^N \beta_i,\quad S=C-\frac{1}{\alpha}mm^\tT,
\end{align}
and 
\begin{equation}
\hat U\in \argmin_{U \in \St(d,n)} G(U). \label{tomini}
\end{equation} 
Here
\begin{align} \label{eq_opt_problem_U_b}
G(U) \coloneqq -\frac{1}{\sigma^2}\trace(U^\tT SU)+\alpha\log(|U^\tT SU|).
\end{align}
\textrm{ii)} If $\sigma$ is learned, we have
\begin{align}
\hat\sigma^2=\tfrac{1}{\alpha(n-d)}\left(\trace(S)-\trace(\hat U^\tT S \hat U)\right),
\end{align}
and $G$ from \eqref{eq_opt_problem_U_b} is replaced by
\begin{align} \label{eq_opt_problem_U_b_sigma}
G(U)\coloneqq(n-d)\log\left(\trace(S)-\trace(U^\tT SU)\right)+\log(|U^\tT SU|).
\end{align}
\end{proposition}

Note that $\alpha$ in the proposition is defined in another way than in the first M-step,
more precisely, the factor $\frac{1}{N}$ is skipped.
Before presenting the proof of the proposition, we give the following remark.

\begin{remark}
By definition of $C$ in \ref{mstep} we have that 
\begin{align}
S
&=\sum_{i=1}^N \beta_i(x_i-\tfrac{1}{\alpha}m)(x_i-\tfrac{1}{\alpha}m)^\tT.
\end{align}
Since $n+1$ of the points $x_i$, $i=1,...,N$ are affinely independent, $S$ is symmetric positive definite. In particular, it holds for $G$ from \eqref{eq_opt_problem_U_b} or \eqref{eq_opt_problem_U_b_sigma} that $G(U)>-\infty$ for any $U\in\St(d,n)$. 
Further, since the function $G$ is continuous and the Stiefel manifold is compact, we can conclude, that $G$ has a global minimizer.
\end{remark}

\begin{proof}[Proof of Proposition \ref{mstep}]
 i) Let $\sigma$ be fixed.
Using \eqref{eq_mult}, we have for fixed  $U$ and $b$, as in the classical GMM, see \eqref{easy} and \eqref{easy2}, 
that the maximizer in \eqref{funct} with respect to
$\mu$ and $\Sigma$ fulfills 
\begin{align}
\mu &= \frac{1}{\alpha} \sum_{i=1}^N \beta_i U^\tT(x_i-b) = \frac{1}{\alpha} (U^\tT m - \alpha U^\tT b),\\
\Sigma &= \frac{1}{\alpha} \sum_{i=1}^N \beta_i \left(U^\tT(x_i-b)-\mu\right) \left(U^\tT(x_i-b)-\mu\right)^\tT\\ 
&=\frac{1}{\alpha} \sum_{i=1}^N \beta_i \left(U^\tT(x_i-\frac{1}{\alpha}m)\right) \left(U^\tT(x_i-\frac{1}{\alpha}m)\right)^\tT
=\frac{1}{\alpha}U^\tT S U.
\end{align}
By Lemma \ref{anders}, the negative objective function in \eqref{funct} is given by
\begin{align}\label{eins}
2 \tilde G(U,b) &= 
G_1(U,b) + G_2(U,b) + \alpha \log(|\Sigma|) + \mathrm{const}, \\
G_1(U,b)&=\tfrac{1}{\sigma^2}\sum_{i=1}^N \beta_i (x_i-b)^\tT (I_n-UU^\tT) (x_i-b)+\alpha(n-d)\log(\sigma^2)
\label{g1}\\
G_2(U,b)&= 
\sum_{i=1}^N \beta_i \left(U^\tT x_i - (U^\tT b+ \mu) \right)^\tT \Sigma^{-1}\left(U^\tT x_i - (U^\tT b+ \mu) \right).
\label{g2}
\end{align}
In the following, we use $\mathrm{const}$ as a generic constant which has values
independent of $\mu,\Sigma,U$ and $b$.
The linear trace operator $\trace: \R^{d,d} \to \R$ fulfills
$x^\tT A y = \trace(A x y^\tT)$ and in particular $x^\tT U U^\tT x = \trace(U^\tT x x^\tT U)$.
Using this property we obtain

\begin{align}
G_2(U,b)&= 
 \trace\Big( \Sigma^{-1}\underbrace{\sum_{i=1}^N \beta_i\left(U^\tT x_i - (U^\tT b+ \mu) \right)\left(U^\tT x_i - (U^\tT b+ \mu) \right)^\tT}_{=\Sigma}\Big)=\trace(I).
 \end{align}
 Thus, the only term in \eqref{eins} which depends on $b$ and $U$ is $G_1$. Further, minimizing $G_1$ is equivalent to minimizing
$$
g_1(U,b)\coloneqq\sum_{i=1}^N \beta_i (x_i-b)^\tT (I_n-UU^\tT) (x_i-b).
$$ 
For fixed $U$, we can minimize $g_1$ with respect to $b$ by setting the gradient to $0$. Since $g_1$ is convex in $b$ this is equivalent for being a global minimizer. This yields
 $$
 0=\sum_{i=1}^N\beta_i(I_n-UU^\tT)(b-x_i)
 $$
 which is equivalent to
 $$
 0=(I_n-UU^\tT)(\alpha b-m).
 $$
 In particular, $b=\frac{1}{\alpha}m$ is a global minimizer of $g_1$ resp. $G_1$, and it is independent of $U$. 
Using this, we get
 $$
 \mu=\frac{1}{\alpha}(U^\tT m-\alpha U^\tT b)=0.
 $$
 Minimizing $G_1$ with respect to $U$ for $b=\frac{1}{\alpha}m$ is equivalent to minimizing
 \begin{align} 
 G_1(U,\frac{1}{\alpha}m )
 &=-\tfrac{1}{\sigma^2} \trace\left(U^\tT S U\right) + \mathrm{const}.
 \end{align}
 Further we have 
 $
 \log\left(\left|\tfrac{1}{\alpha}U^\tT SU\right|\right)=\log(|U^\tT SU|)+\mathrm{const}
 $.
 Thus, by combining the above computations, we get that minimizing \eqref{eins} with respect to $U$ is equivalent to minimizing
 \begin{align}
 G(U)&=-\frac{1}{\sigma^2}\trace\left(U^\tT S U\right) +\alpha\log(|U^\tT SU|).
 \end{align}

ii)
Now consider the case, where $\sigma$ is learned.
Again by \eqref{eq_mult}, the maximizer in \eqref{funct} with respect to $\sigma$ is given by the maximizer of 
$$
\sum_{i=1}^N\beta_i\big(-\tfrac1{2\sigma^2}\|(I_n-UU^\tT)(x_i-b)\|^2-(n-d)\log(\sigma)\big).
$$
By setting the derivative to zero, one obtains, that
$$
\sigma^2=\tfrac1{\alpha(n-d)}\sum_{i=1}^N\beta_i(x_i-b)^\tT(I_n-UU^\tT)(x_i-b).
$$
Then the function in \eqref{g1} modifies to
\begin{align}\label{G1}
G_1(U,b)=\alpha(n-d)\log\Big(\sum_{i=1}^N \beta_i (x_i-b)^\tT (I_n-UU^\tT) (x_i-b)\Big)+\mathrm{const}.
\end{align}
Now the monotonicity of the logarithm implies that minimizing $G_1$ is again equivalent to minimizing $g_1$. 
Hence, as in case i) we get
$b=\frac{1}{\alpha}m$ is a global minimizer of $g_1$ resp. $G_1$, and it is independent of $U$. 
Using this, we obtain
$$
\mu=\frac{1}{\alpha}(U^\tT m-\alpha U^\tT b)=0\quad\text{and}\quad\sigma^2=\tfrac{1}{\alpha(n-d)}\left(\trace(S)-\trace(U^\tT S U)\right).
$$
By \eqref{G1}, minimizing $G_1$ with respect to $U$ for $b=\frac{1}{\alpha}m$ is equivalent to minimizing
\begin{align} 
G_1(U,\frac{1}{\alpha}m )
&=\alpha(n-d)\log\left(\trace(S)-\trace(U^\tT SU)\right) + \mathrm{const},
\end{align}
such that minimizing \eqref{eins} with respect to $U$ is equivalent to minimizing
\begin{align}
G(U)&=(n-d)\log\left(\trace(S)-\trace(U^\tT SU)\right)+\log(|U^\tT SU|).
\end{align}

\end{proof}

By Proposition \ref{mstep}, the M-Step of Algorithm \ref{alg_em_gmm_pca}
reduces for $k=1,...,K$ to the computation
\begin{align}
\alpha_k^{(r+1)}
& = 
\frac{1}{N} \sum_{i=1}^N \beta_{i,k}^{(r)},
\\
m_k 
&= 
\sum_{i=1}^N \beta_{i,k} x_i, \; 
C_k = \sum_{i=1}^N \beta_{i,k} x_i x_i^\tT,
\\
(U_k^{(r+1)},b_k^{(r+1)})
&=
\argmin_{U\in \SPD(d,n),b \in \mathbb R^n} G_k(U,b)
\qquad \mathrm{with} \; G_k \; \mathrm{in} \; \eqref{eq_opt_problem_U_b},
\\
\mu_k^{(r+1)}    
&= 
\frac{1}{ N\alpha_k^{(r+1)} } 
\big( U_k^{(r+1)} \big)^\tT 
\left( m_k - N \alpha_k^{(r+1)} b_k^{(r+1)} \right),
\\
S_k &=
 C_k - m_k \left(b_k^{(r+1)} \right)^\tT 
- b_k^{(r+1)} m_k^\tT + N \alpha_k^{(r+1)} b_k^{(r+1)} 
\left(b_k^{(r+1)}\right)^\tT\\
\Sigma_k^{(r+1)} &= 
\frac{1}{N\alpha_k^{(r+1)} } \big(U_k^{(r+1)} \big)^\tT
\,
S_k \,
U_k^{(r+1)} 
\end{align}
Note that the large data set $\mathcal X$ is involved in the computation of $m_k$ and $C_k$, but it does not influence 
the computational time for minimizing the
$G_k$, $k=1,\ldots,K$. Indeed, the E-Step of Algorithm \ref{alg_em_gmm_pca}
will be the most time consuming one. 

\subsection{PALM for Minimizing $G$} \label{sec:alg-palm}

To minimize $G$ in \eqref{eq_opt_problem_U_b}
we propose to use the Proximal alternating linearized minimization
(PALM) \cite{BST2014}, resp. its accelerated version iPALM \cite{PS2016},
where the 'i' stands for inertial.
As a special case the PALM algorithm can be applied to functions of the form 
\begin{align}
F(x) = H(x)+f(x)\label{eq_PALM_general}
\end{align}
where $H\in C^1(\R^{d})$ and a
lower semi-continuous function 
$f\colon \R^{d} \to (-\infty,\infty]$. 
It is based on the computation of so-called proximal operators.
For a proper and lower semi-continuous function $f\colon\R^d\to (-\infty,\infty]$ 
and $\tau>0$ the \emph{proximal mapping} $\prox_\tau^f\colon\R^d\to\mathcal{P}(\R^d)$ is defined by
$$
\prox_\tau^f(x)=\argmin_{y\in\R^d} \left\{ {\tfrac{\tau}2\|x-y\|^2+f(y)} \right\},
$$
where $\mathcal{P}(\R^d)$ denotes the power set of $\R^d$.

Starting with an arbitrary $x^{(0)}$ PALM performs the iterations
\begin{align}\label{iteration_scheme}
x^{(r+1)}&\in \prox_{\tau^{(r)}}^f \left(x^{(r)}-\tfrac{1}{\tau^{(r)}}\nabla H(x^{(r)}) \right)
\end{align}
Further, iPALM is detailed in Algorithm \ref{iPALM}. Indeed, we have applied the iPALM algorithm in our numerical examples.
However, although we observed convergence of the iterates numerically, we have not proved
convergence theoretically so far.
Alternatively, we could apply the PALM algorithm which is slightly slower. 
Note again, that the E-Step of the algorithm is the most time consuming one.

\begin{algorithm}[!ht]
\caption{iPALM}\label{iPALM}
\begin{algorithmic}
\State Input: $\alpha^{(r)},\beta^{(r)}$
initialization $x^{(1)}$, $x^{(0)}$
\For {$r=1,2,...$} until a convergence criterion is reached
\State
\begin{align}
y^{(r)}&=x^{(r)}+\alpha^{(r)}(x^{(r)}-x^{(r-1)}),\\
z^{(r)}&=x^{(r)}+\beta^{(r)}(x^{(r)}-x^{(r-1)}),\\
x^{(r+1)}&\in \prox_{\tau^{(r)}}^f(y^{(r)}-\tfrac{1}{\tau^{(r)}}\nabla H(z^{(r)})),\\
\end{align}
\EndFor
\end{algorithmic}
\end{algorithm}

In the following, we give details on PALM for our setting.
For our problem \eqref{tomini}, we choose $f(U) \coloneqq \iota_{\St(d,n)}$ and
\begin{equation} \label{H_neu}
H(U) \coloneqq G(U)\eta(\|I_d-U^TU\|_F^2),
\end{equation}
where 
$$
\eta(x) \coloneqq
\begin{cases}
1,&$if $x\in (-\rho,\rho),\\
\exp(-\frac{\rho}{\rho-(|x|-\rho)^2}),&$if $x\in (-2\rho,-\rho]\cup[\rho,2\rho),\\
0, &$otherwise.$
\end{cases}
$$
is a smooth cutoff function of the interval $(-\rho,\rho)$ for some $\rho>0$.
Then, the iteration scheme reads as
\begin{align} 
U^{(r+1)}&\in \Pi_{\St(d,n)}(U^{(r)}-\tfrac{1}{\tau^{(r)}}\nabla H(U^{(r)}))\label{PALM_eins}
\end{align}
where $\Pi_{\St(d,n)}$ denotes the orthogonal projection onto the Stiefel manifold. 

\begin{remark} (Projection onto Stiefel manifolds)
Concerning this orthogonal projection, 
it is well known \cite{H1989}, that for a matrix $A\in\R^{n, d}$, 
the projection $\Pi_{\St(d,n)}(A)$ is given by the orthonormal polar factor $W$ from the polar decomposition 
$$
A=W M,\quad W\in\St(d,n),\quad M\in\SPD(d).
$$
Further, this orthonormal polar factor can be computed by $W=UV$, where $A=U\Sigma V$ is the singular value decomposition of $A$, 
see \cite{H1989}.
The authors of \cite{HS1990} propose to use the so-called Schulz-iteration
$$
X_{k+1}=X_k(I+\tfrac12(I-X_k^\tT X_k))
$$
with $X_0=A$ for computing the orthonormal polar factor of a full rank matrix $A$.
Unfortunately, the convergence of this iteration requires that $\|I-A^\tT A\|_F<1$,
which is usually not fulfilled in our case.
\end{remark}

Note that for any $r\in \mathbb N$, the matrix $U^{(r)}$ belongs to the Stiefel manifold, 
such that $\eta(\|I_d-U^\tT U\|_F)=1$ in a neighborhood of $U^{(r)}$. 
Thus, we can replace the gradient with respect to $H$ by the gradient with respect to $G$ in \eqref{PALM_eins}.
Then the iteration scheme reads as
\begin{align} 
U^{(r+1)}&\in P_{\St(d,n)}(U^{(r)}-\tfrac{1}{\tau^{(r)}}\nabla G(U^{(r)})), \label{PALM_drei}\\
\end{align}
In particular, we do not need to choose the $\rho$ explicitly within our algorithm.

To show convergence of the algorithm, we need the following two lemmas.

\begin{lemma}\label{lemma_1}     
Let $H$ be defined by \eqref{H_neu}.
Then the function $\nabla H$ is globally Lipschitz continuous.
\end{lemma}

\begin{proof}
The function $H$ is twice continuously differentiable and zero outside of a compact set. 
Hence the second order derivative is bounded and $\nabla_U H(\cdot,b)$ is globally Lipschitz continuous.
\end{proof}

Further, let us recall the notation of Kurdyka-{\L}ojasiewicz functions.
For $\delta\in(0,\infty]$, we denote by $\Phi_\delta$ the set of all concave continuous functions $\phi\colon[0,\delta)\to\R_{\geq 0}$ which fulfill the following properties:
\begin{enumerate}[(i)]
\item $\phi(0)=0$.
\item $\phi$ is continuously differentiable on $(0,\delta)$.
\item For all $s\in(0,\delta)$ it holds $\phi'(s)>0$.
\end{enumerate}

For a proper and lower semicontinuous function $\gamma\colon\R^d\to(-\infty,+\infty]$ denote by $\partial \gamma$ the subdifferntial of $\gamma$.

\begin{definition}[Kurdyka-{\L}ojasiewicz property]
A  proper, lower semicontinuous function $\gamma\colon\R^d\to(-\infty,+\infty]$
has the Kurdyka-{\L}ojasieweicz (KL) property at $\bar u\in\mathrm{dom}\,\partial\gamma=\{u\in\R^d:\partial\gamma\neq\emptyset\}$
if there exist $\delta\in(0,\infty]$, a neighborhood $U$ of $\bar u$ and a function $\phi\in\Phi_\delta$, such that for all
$$
u\in U\cap\{v\in\R^d:\gamma(\bar u) < \gamma(v) < \gamma(\bar u)+\delta\},
$$
it holds
$$
\phi'(\gamma(u)-\gamma(\bar u))\mathrm{dist}\,(0,\partial\gamma(u))\geq 1.
$$
We say that $\gamma$ is a KL function, if it satisfies the KL property in each point $u\in\mathrm{dom}\,\partial\gamma$.
\end{definition}

\begin{lemma} \label{lemma_2}
The function $H$ defined in \eqref{H_neu} is a KL function.
\end{lemma}
\begin{proof}
The functions $G$ and $\eta$ are sums, products, quotients and concatenations of real analytic functions. Thus, also $H$ is a real analytic function. This implies that it is a KL function, see \cite[Remark 5]{AB2009} and \cite{L1963, L1993}.
\end{proof}

The following theorem follows directly from \cite[Lemma 3, Theorem 1]{BST2014}. 

\begin{theorem}[Convergence of PALM] \label{thm:PALM_convergence}
Let $F\colon \R^{d}\to(-\infty,\infty]$ be given by \eqref{eq_PALM_general} and let $\nabla H$ be globally $L$-Lipschitz continuous.
Let $(x^{(r)})_r$ be the sequence generated by PALM, 
where the step size parameters fulfill 
$$\tau^{(r)}  \ge  \gamma L$$
for some $\gamma >1$. 
Then, for $\eta \coloneqq (\gamma-1)L$,
 the sequence $(F(x^{(r)}))_r$ is nonincreasing and 
$$
\tfrac{\eta}{2}\|x^{(r+1)}-x^{(r)})\big\|_2^2 \leq F(x{(r)})-F(x^{(r+1)}).
$$
If $F$ is in addition a KL function and the sequence $(x^{(r)})_r$ is bounded, then it converges to a critical point of $F$.
\end{theorem}

By Lemma \ref{lemma_1} and \ref{lemma_2} and the fact that $G$ coincides with $H$ in a neighborhood of the Stiefel manifold we obtain the following corollary. 

\begin{corollary}\label{cor:EM-M}
Let $(U^{(r)})_r$ be generated by \eqref{PALM_drei} 
with $\tau^{(r)}\geq \gamma L$, 
where $L$ is the Lipschitz constant of $\nabla H$ and $\gamma>1$.
Consider the sequence generated by PALM with \eqref{PALM_drei}.
Then, the sequence $(G(U^{(r)}))_r$ is monotone decreasing and the sequence $(U^{(r)})_r$ converges to a critical point of $G$.
\end{corollary}

\section{Superresolution} \label{sec:super}
In this section, we adapt the superresolution method proposed by Sandeep and Jacob \cite{SJ2016} to our 
PCA-GMM model.
The method works in two steps.

\paragraph{1. Learning the PCA-GMM}
For given low resolution patches $x_{L,i}\in\R^{\tau^2}$ 
of an image and their higher resolution counterparts 
$x_{H,i}\in\R^{q^2\tau^2}$, $q \in \mathbb N$, $q>2$,
$i=1,\ldots,N$
we learn a PCA-GMM based on the data 
$x_i = \left(\begin{array}{c}x_{H,i}\\ x_{L,i}\end{array}\right) \in \R^n$, where 
$n= (q^2+1)\tau^2$,
by Algorithm \ref{alg_em_gmm_pca}.
This  provides us with parameters
$(\mathbf U, \mathbf b, \alpha,\mathbf \mu,\mathbf \Sigma)$
of the reduced $d$-dimensional GMM.
Using these parameters, we compute the parameters of the corresponding
high-dimensional mixture model $(\alpha,\tilde \mu_k,\tilde \Sigma_k)$, $k=1,\ldots,K$,
where $\mu_k$ and $\Sigma_k$ are defined as in \eqref{eq_mu_new} and \eqref{eq_sig_new}.
In the following, we use the notations 
$\tilde \mu_k=\left(\begin{array}{c}
\tilde \mu_{H,k}\\
\tilde \mu_{L,k}
\end{array}\right)$ 
and 
$\tilde \Sigma_k
=
\left(\begin{array}{cc}
\tilde \Sigma_{H,k}&\tilde \Sigma_{HL,k}\\
(\tilde \Sigma_{HL,k})^\tT&\tilde \Sigma_{L,k}\end{array}\right)$.

\paragraph{2. Estimation of high resolution patches by MMSE}
Now we want to improve the resolution of a given low resolution patch $x_L \in \R^{\tau^2}$.
First, we select the component $k^*$, such that the likelihood that $x_L$ belongs to the $k^*$-th component is maximal, 
i.e., we compute
$$
k^* = \argmax_{k=1,...,K} \alpha_k f(x_L|\tilde \mu_{L,k},\tilde \Sigma_{L,k}).
$$
Then we estimate the high resolution patch $x_H \in\R^{q^2\tau^2}$ 
as the minimum mean-square estimator (MMSE). The following remark briefly review this estimator.

\begin{remark} (MMSE)
Given a random variable $Y: \Omega \rightarrow \mathbb R^d$ 
in a probability space $(\Omega, \mathcal A, \mathbb P)$, 
we wish to estimate a random variable $X:\Omega \rightarrow \mathbb R^d$, 
i.e., we seek an estimator $T\colon \R^d\to \R^d$ such that
$\hat{X} = T(Y)$ approximates $X$. 
A common quality measure for this task is the \emph{mean square error} 
$\E\norm{X-T(Y)}_{2}^2$, 
which gives rise to the definition of the \emph{minimum mean square estimator} 
\begin{equation} \label{mmse_1}
T_{\text{MMSE}}\in \argmin_{T}\E\|X-T(Y)\|_2^2.
\end{equation}
Under weak additional regularity assumptions on the estimator $T$, 
the Lehmann-Scheff\'e theorem~\cite{LS50}
states that the general solution of the minimization problem \eqref{mmse_1} is given by  
$$T_{\text{MMSE}}(Y)= \E(X|Y). $$ 
In general, it is not possible to give an analytical expression of the MMSE estimator 
$T_{\text{MMSE}}$. 
An exception are Gaussian random variables:
if $X$ and $Y$ are jointly normally distributed, i.e.,
\begin{equation*}
\begin{pmatrix}
X\\Y
\end{pmatrix}
\sim \mathcal N\Biggl(\begin{pmatrix} \mu_X\\ \mu_Y
\end{pmatrix},
\begin{pmatrix}
\Sigma_X & \Sigma_{XY}\\
\Sigma_{YX} & \Sigma_Y
\end{pmatrix}\Biggr),
\end{equation*}
then
the conditional distribution of $X$ given $Y=a$ is normally as well and reads as
\begin{equation*}
(X|Y=a) \sim \mathcal N\bigl(\mu_{X|Y},\Sigma_{X|Y}  \bigr), 
\end{equation*}
where
\begin{equation}
\mu_{X|Y} = \mu_X + \Sigma_{XY} \Sigma^{-1}_Y (a-\mu_Y),\qquad \Sigma_{X|Y}= \Sigma_X-\Sigma_{XY}\Sigma^{-1}_Y\Sigma_{YX}.\label{eq_MMSE_mean_cov}
\end{equation}
As a consequence we obtain for normally distributed random variables the MMSE estimator 
\begin{equation} \label{mmse_est}
T_{\mathrm{MMSE}}(Y)= \E(X|Y) = \mu_X + \Sigma_{XY} \Sigma^{-1}_Y (Y-\mu_Y) . \qquad
\end{equation}
\end{remark}

In our superresolution task, we assume that the vector 
$\left(\begin{array}{c}x_H\\x_L\end{array}\right)$ 
is a realization of a random variable
${\left(\begin{array}{c}
X_H\\X_L\end{array}\right)
\sim \mathcal N(\tilde \mu_{k^*},\tilde \Sigma_{k^*})}
$. 
Then, by \eqref{mmse_est}, the MMSE can be computed as
\begin{equation} \label{mmse_super}
x_H = \tilde\mu_{H,k^*} + \tilde\Sigma_{HL,k^*} (\tilde\Sigma_{L,k^*} )^{-1}(x_L-\tilde\mu_{L,k^*}).
\end{equation}

\paragraph{3. Reconstruction of the high resolution image by patch averaging}

We estimate for any patch in the low resolution image the corresponding high resolution patch.
Once, we have estimated the high resolution patches, we compute an estimate of the high resolution image in the following way:

Let $x_H=(x_{k,l})_{k,l=1}^{q\tau}\in\R^{q\tau,q\tau}$ be a two-dimensional high resolution patch. Now, we assign to each pixel $x_{k,l}$ the weight 
$$
w_{k,l}\coloneqq\exp\Big(-\tfrac{\gamma}{2} \big((k-\tfrac{q\tau+1}{2})^2+(l-\tfrac{q\tau+1}{2})^2\big)\Big).
$$
After that, we add up for each pixel in the high resolution image the corresponding weighted pixel values and normalize the result by dividing by the sum of the weights.

\section{Numerical Results} \label{sec:num}
In this section, we demonstrate the performance of our algorithm by two- and three-dimensional examples,
where we mainly focus on material data which provided the original motivation for this  work.
More precisely, in the frame of the ITN MUMMERING, 
a series of multi-scale 3D images has been acquired by synchrotron micro-tomography at the SLS beamline TOMCAT.
Materials of two samples were selected to provide 3D images having diverse levels of complexity:
\begin{itemize}
\item[-]
 The first one is a sample of Fontainebleau sandstone ("FS"), 
a natural rock rather homogeneous and commonly used in the oil industry for flow experiments. 
\item[-] The second one is a composite ("SiC Diamonds") obtained by microwave sintering of silicon and diamonds, 
see \cite{vaucher2007line}.
\end{itemize}
Sections of the corresponding 3D images are given the first two columns of Figure \ref{fig_learn_ims}.

All implementations were done in Python and Tensorflow and they can be  parallelized on a GPU. 
We run all our experiments on a Lenovo ThinkStation with Intel i7-8700 6-Core processor with 32GB RAM and NVIDIA GeForce GTX-2060 Super GPU. 
The code is available online\footnote{\url{https://github.com/johertrich/PCA_GMMs}}.

For the implementation of PALM and iPALM, 
we use the implementation framework from \cite{HS2020}\footnote{\url{https://github.com/johertrich/Inertial-Stochastic-PALM}}. 
As suggested in \cite{PS2016} we set the extrapolation factors $\gamma_1^{(r)}=\gamma_2^{(r)}=\frac{r-1}{r+2}$ and choose $\tau_1^{(r)}=\frac{1}{\tilde L_1(b^{(r)}}$ and $\tau_2^{(r)}=\frac{1}{\tilde L_2(U^{(r+1)})}$, where $\tilde L_1(b^{(r)})$ and $\tilde L_2(U^{(r+1)})$ are estimates of the Lipschitz constant of $\nabla_U G(\cdot,b^{(r)})$ and $\nabla_b G(U^{(r+1)},\cdot)$.\\
We generate pairs of high and low resolution images using the following superresolution operator:

\paragraph{Generation of the test examples.}
For convenience, we describe the generation in 2D. The 3D setting is treated in a similar way.
We use the operator $A$ from the implementation of \cite{PDDN2019}\footnote{\url{https://github.com/pshibby/fepll_public}}. 
This operator consists of a blur operator $H$ and a downsampling operator $S$. 
The blur operator is given by a convolution with a Gauss kernel with standard deviation $0.5$. 
For the downsampling operator $S$ we use the discrete Fourier transform (DFT). 
Given an image $x\in\R^{m,n}$ the two-dimensional DFT 
is defined by $\mathcal F_{m,n}\coloneqq \mathcal F_n\otimes \mathcal F_m$, 
where $\mathcal F_n=(\exp(-2\pi i k l/n))_{k,l=0}^{n-1}$. Now, the downsampling operator $S\colon\R^{m,n}\to\R^{m_2,n_2}$ 
is given by
$$
S=\frac{m_2n_2}{mn}\mathcal F_{m_2,n_2}^{-1} D \mathcal F_{m,n},
$$
where for $x\in\R^{m,n}$ the $(i,j)$-th entry of $D(x)$ is given by
$$
\begin{cases}
x_{i,j}, &$if $i\leq\frac{m_2}{2}$ and $j\leq\frac{n_2}{2},\\
x_{i+m-m_2,j}, &$if $i>\frac{m_2}{2}$ and $j\leq\frac{n_2}{2},\\
x_{i,j+n-n_2}, &$if $i\leq\frac{m_2}{2}$ and $j>\frac{n_2}{2},\\
x_{i+m-m_2,j+n-n_2}, &$if $i>\frac{m_2}{2}$ and $j>\frac{n_2}{2}.
\end{cases}
$$

For a given high resolution image $x$, we now generate the low resolution image $y$ by $y=Ax+\epsilon$, 
where $\epsilon$ is a realization of white Gaussian noise with standard deviation $0.02$.

\paragraph{Initialization of the EM algorithms.}
Since the negative log-likelihood function is non-convex and admits many critical points, EM algorithms for GMMs are very sensitive with respect to the initialization. 
For example this can be seen by considering 
the case when $\vartheta_k^{(r)}=\vartheta_l^{(r)}$, $k,l=1,...,K$ for some $r\in\N$.
Then we obtain that $\beta_{i,k}^{(r)}=\alpha_k$ and consequently $\vartheta_k^{(r+1)}=\vartheta_l^{(r+1)}$, $k,l=1,...,K$.
The same effect appears for PCA-GMMs and HDDC.
Consequently the initialization of the EM algorithms is of great importance.
For our numerical examples, we initialize the GMMs as follows.
We set $\alpha^{(0)}_k=\tfrac1K$, for $k=1,...,K$. 
For initializing the means, we choose randomly $K$ distinct data points $\mu_1,...,\mu_K$ from our training data $x_1,...,x_N$. 
Finally, we choose for each $k=1,...,K$ the $M$ points $y_1,...,y_M$ from $x_1,...,x_N$ which are the closest ones to $\mu_k$ and initialize the covariances by $\Sigma_k=\tfrac1M\sum_{i=1}^My_iy_i^\tT$.
The number $M$ is chosen according to the dimension $n$ of the data. In our examples, we use $M=2n$.

We initialize the PCA-GMMs and HDDC by taking the initialization for GMMs, running one E-Step from the EM algorithm for GMMs followed by the M-step of the PCA-GMMs or HDDC, respectively.

\paragraph{Choice of $\sigma$ and $K$.}
The PCA-GMM model depends heavily on the choice of the parameter $\sigma$. 
As pointed out  in Subsection \ref{sec:alg-em}, this parameter could be learned from the data.
However, the forward model for the low resolution images $y=Ax+\epsilon$ for some (unknown) superresolution operator $A$, the high resolution image $x$ and noise $\epsilon$ suggest to choose the $\sigma$ according to the standard deviation of the noise $\epsilon$.
Note, that in our experiments, the low resolution images are artificially generated by applying a downsampling operator and adding some noise.
Consequently, the standard deviation of $\epsilon$ is known. 
Nevertheless, if the noise level is unknown, it could be estimated very accurately from the data based on homogeneous area detection as done, e.g., in \cite{HHLS2021,SDA2015}.

In practice, it can be unstable to estimate the standard deviation of the noise within the optimization of the mixture model, since this requires that the image patches belong exactly (and not only approximately) to a dimensionality reduced GMM with $K$ components, which is an unrealistic assumption.
Therefore, it can be beneficial and quite more accurate to estimate the standard deviation of the noise a priori.
In particular, if the standard deviation of the noise is known, fixing $\sigma$ can be the better approach.

Note that the noise with standard deviation $\sigma$ within the superresolution model does not necessarily imply that the eigenvalues of the covariance matrices in the mixture model are greater than or equal to $\sigma^2$ (which is assumed for HDDC), since the noise is only applied to the low resolution images.

Also the number of components $K$ of the mixture models can have a large impact on the results. For superresolution, a detailed comparison of the prediction quality for different choices of $K$ was done by Sandeep and Jacob in \cite{SJ2016}. 
They observed that the benefit of taking more than $100$ components in the GMM is usually very small. 
Therefore, we take $K=100$ components for all mixture model in our numerical examples.

\paragraph{Comparison of the computation times.}
Note that there already exist implementations of HDDC by some of the authors of \cite{BGS2006}.
However, to provide a fair comparison of the execution times, we reimplement the EM algorithm for HDDC in Python and Tensorflow, such that it supports GPU parallelization.
Further, note that we compute the updates of $\alpha$, $m$ and $C$ simultaneously to the E-step such that the corresponding execution time is contained in the E-step, even though the updates technically belong to the M-step.
This has the advantage that we have to iterate only once over the whole data set and enables a better parallelization.
We implemented this optimization of the order of computation for all of the models (GMM, PCA-GMM and HDDC) analogously.

\paragraph{2D-Data.}
For estimating the parameters of the mixture models, we use the upper left quarter of the image as in the top row of 
Figure \ref{fig_learn_ims}. 
As ground truth for the reconstruction we use the whole images as in the bottom row. 
The images in the left and middle columns are the middle slices of the material data "FS" and "SiC Diamonds". 
The high resolution images have a size of $2560\times 2560$.
The right column contains the \texttt{goldhill} image, which has the size $512\times 512$. 

\begin{figure}
\begin{subfigure}[t]{0.33\textwidth}
\centering
\includegraphics[width=\textwidth]{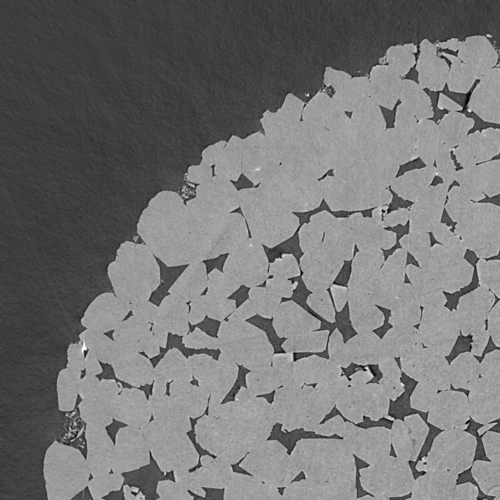}
\end{subfigure}\hfill
\begin{subfigure}[t]{0.33\textwidth}
\centering
\includegraphics[width=\textwidth]{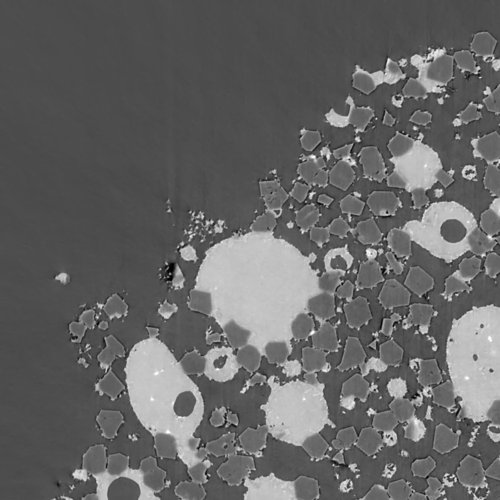}
\end{subfigure}\hfill
\begin{subfigure}[t]{0.33\textwidth}
\centering
\includegraphics[width=\textwidth]{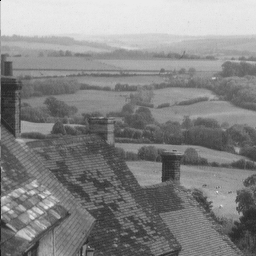}
\end{subfigure}\hfill
\begin{subfigure}[t]{0.33\textwidth}
\centering
\includegraphics[width=\textwidth]{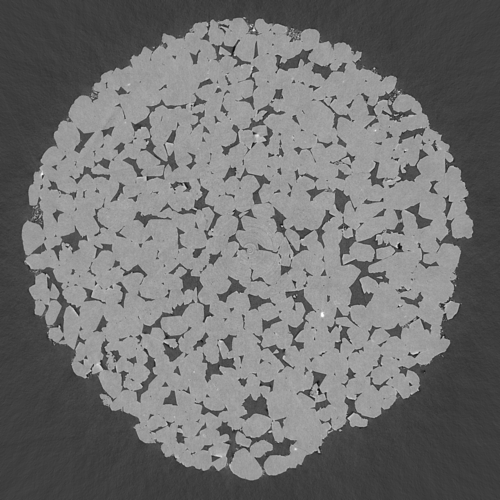}
\end{subfigure}\hfill
\begin{subfigure}[t]{0.33\textwidth}
\centering
\includegraphics[width=\textwidth]{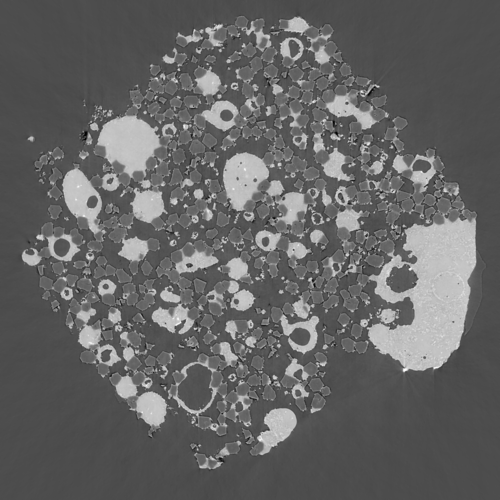}
\end{subfigure}\hfill
\begin{subfigure}[t]{0.33\textwidth}
\centering
\includegraphics[width=\textwidth]{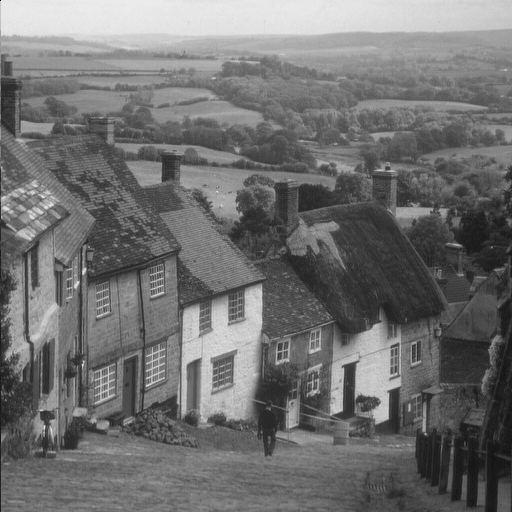}
\end{subfigure}\hfill
\caption{Top: Images for estimating the mixture models. Bottom: Ground truth for reconstruction. 
First column: Material "FS", second column: Material "SiC Diamonds", third column: \texttt{goldhill} image.}
\label{fig_learn_ims}
\end{figure}

We estimate the parameters of a GMM and of our PCA-GMM as described in the previous sections. 
First, we fix the parameter $\sigma$ in Algorithm \ref{alg_em_gmm_pca} as the standard deviation of the noise on the low dimensional image (i.e. $\sigma=0.02$). 
Second, we consider the case when $\sigma$ is learned from the data and finally we compare our results with HDDC \cite{BGS2006}.
Each mixture model has $K=100$ classes.
We use the magnification factors $q\in\{2,4\}$ and the patch size $\tau=4$ for the low resolution patches. This corresponds to a patch size of $q\tau=8$ or $q\tau=16$ respectively for the high resolution images. 
For the material images, this leads to $N\approx 400000$ patches for $q=2$ and $N\approx 100000$ for $q=4$. Using the \texttt{goldhill} image, we get $N\approx 15000$ patches for $q=2$ and $N\approx 3700$ patches for $q=4$.
We reduce the dimension of the pairs of high and low resolution patches from $n=(q^2+1)\tau^2=80$ or $n=(q^2+1)\tau^2=272$ respectively to $d$ for $d\in\{4,8,12,16,20\}$.
After estimating the mixture models, we use the reconstruction method from \cite{SJ2016} as described in the previous section to reconstruct the ground truth from the artificially downsampled images.
The resulting PSNRs are given in Table~\ref{tab_results}. As a reference we also measure the PSNR of the bicubic interpolation.
 The average execution times for one E-step and one M-step are given in Table~\ref{tab_times}.
Figure \ref{fig_zoom} shows some small areas of the high resolution images, low resolution images and the corresponding reconstructions for GMMs and PCA-GMM with $d=12$ and $d=20$. The result with $d=12$ for PCA-GMM is already almost as good as GMM, whereas the dimension of the patches was reduced by a factor between 4 and 22 (depending on the case).
Further, we observed that the dimensionality reduction reduces the execution time of the E-step significantly. On the other hand, the execution time of the M-step is larger than those in the GMM for all dimension reduced models due its higher complexity.
Comparing the different dimensionality reduced models, we observe that the PCA-GMM with fixed $\sigma$ gives significantly better results than the other models, while HDDC achieves the fastest M-step due to the closed-form updates.
However, compared to the execution time of the E-step, this advantage seems to be negligible for large data sets as, e.g., the patches from the FS and SiC Diamonds image.

\begin{table}[htp]
\begin{center}
\begin{tabular}{c|c| c c c | c c c}
&&\multicolumn{3}{c|}{Magnification factor $q=2$}&\multicolumn{3}{c}{Magnification factor $q=4$}\\\hline
&$d$&FS&Diamonds&Goldhill&FS&Diamonds&Goldhill\\\hline
bicubic&-&$30.57$&$30.67$&$28.99$&$25.27$&$25.19$&$24.66$\\
GMM&-&$35.49$&$37.21$&$31.63$&$30.69$&$30.74$&$27.80$\\\hline
\multirow{5}{6em}{\centering PCA-GMM, $\sigma=0.02$}
&$20$&$35.44$&$37.24$&$31.25$&$30.75$&$30.74$&$27.64$\\
&$16$&$35.42$&$37.22$&$31.25$&$30.74$&$30.62$&$27.59$\\
&$12$&$35.47$&$37.13$&$31.18$&$30.67$&$30.48$&$27.55$\\
&$8$&$35.32$&$36.69$&$31.00$&$30.46$&$30.16$&$27.38$\\
&$4$&$34.69$&$35.23$&$30.42$&$29.78$&$29.24$&$26.89$\\\hline
\multirow{5}{6em}{\centering PCA-GMM, learned $\sigma$}
&$20$&$35.22$&$37.06$&$31.27$&$30.43$&$30.51$&$27.66$\\
&$16$&$35.14$&$37.01$&$31.14$&$30.34$&$30.31$&$27.51$\\
&$12$&$34.95$&$36.54$&$30.94$&$30.13$&$29.84$&$27.33$\\
&$8$&$34.43$&$35.47$&$30.54$&$29.62$&$29.08$&$26.88$\\
&$4$&$32.74$&$33.41$&$29.69$&$28.51$&$27.75$&$26.16$\\\hline
\multirow{5}{6em}{\centering HDDC \cite{BGS2006}}
&$20$&$35.35$&$37.12$&$31.35$&$30.54$&$30.63$&$27.73$\\
&$16$&$35.31$&$37.10$&$31.25$&$30.47$&$30.48$&$27.62$\\
&$12$&$35.24$&$36.64$&$31.08$&$30.27$&$30.08$&$27.40$\\
&$8$&$34.76$&$35.66$&$30.76$&$29.80$&$29.34$&$27.00$\\
&$4$&$33.46$&$33.86$&$29.93$&$28.61$&$27.99$&$26.37$
\end{tabular}
\end{center}
\caption{PSNRs of the reconstructions of artificially downsampled 2D images using either bicubic interpolation, a GMM, PCA-GMM for different choices of $d$ or HDDC. The magnification factor is set to $q\in\{2,4\}$.
PCA-GMM produces results almost as good as GMM, with a much lower dimensionality.
}
\label{tab_results}
\end{table}

\begin{table}[htp]
\begin{center}
\begin{tabular}{c|c| c c | c c | c c}
&&\multicolumn{6}{c}{Magnification factor $q=2$, i.e.\ dimension $n=80$}\\\hline
&&\multicolumn{2}{c|}{FS, $N=405769$}&\multicolumn{2}{c|}{Diamonds, $N=405769$}&\multicolumn{2}{c}{Goldhill, $N=15625$}\\
&$d$&E-step&M-step&E-step&M-step&E-step&M-step\\\hline
GMM&-&$10.91$&$0.06$&$10.91$&$0.06$&$0.44$&$0.06$\\\hline
\multirow{3}{6em}{\centering PCA-GMM, $\sigma=0.02$}
&$20$&$\n7.25$&$0.74$&$\n7.42$&$0.57$&$0.28$&$0.54$\\
&$12$&$\n6.58$&$0.59$&$\n6.53$&$0.51$&$0.25$&$0.46$\\
&$4$&$\n6.18$&$0.56$&$\n6.17$&$0.52$&$0.24$&$0.48$\\\hline
\multirow{3}{6em}{\centering PCA-GMM, learned $\sigma$}
&$20$&$\n7.28$&$0.54$&$\n7.41$&$0.54$&$0.28$&$0.54$\\
&$12$&$\n6.59$&$0.47$&$\n6.53$&$0.45$&$0.25$&$0.47$\\
&$4$&$\n6.20$&$0.47$&$\n6.17$&$0.44$&$0.24$&$0.51$\\\hline
\multirow{3}{6em}{\centering HDDC \cite{BGS2006}}
&$20$&$\n7.27$&$0.27$&$\n7.44$&$0.27$&$0.28$&$0.26$\\
&$12$&$\n6.64$&$0.26$&$\n6.64$&$0.26$&$0.25$&$0.26$\\
&$4$&$\n6.27$&$0.27$&$\n6.23$&$0.26$&$0.24$&$0.26$
\end{tabular}\\\vspace{.5cm}
\begin{tabular}{c|c| c c | c c | c c}
&&\multicolumn{6}{c}{Magnification factor $q=4$, i.e.\ dimension $n=272$}\\\hline
&&\multicolumn{2}{c|}{FS, $N=100489$}&\multicolumn{2}{c|}{Diamonds, $N=100489$}&\multicolumn{2}{c}{Goldhill, $N=3721$}\\
&$d$&E-step&M-step&E-step&M-step&E-step&M-step\\\hline
GMM&-&$17.15$&$0.06$&$17.11$&$0.06$&$0.90$&$0.06$\\\hline
\multirow{3}{6em}{\centering PCA-GMM, $\sigma=0.02$}
&$20$&$\n8.65$&$3.54$&$\n8.68$&$2.03$&$0.44$&$1.83$\\
&$12$&$\n8.17$&$2.73$&$\n8.15$&$1.99$&$0.42$&$1.75$\\
&$4$&$\n7.95$&$2.10$&$\n7.94$&$2.49$&$0.41$&$1.91$\\\hline
\multirow{3}{6em}{\centering PCA-GMM, learned $\sigma$}
&$20$&$\n8.65$&$1.92$&$\n8.70$&$1.83$&$0.44$&$1.87$\\
&$12$&$\n8.17$&$1.99$&$\n8.17$&$1.74$&$0.42$&$1.74$\\
&$4$&$\n7.94$&$2.17$&$\n7.93$&$1.76$&$0.41$&$1.72$\\\hline
\multirow{3}{6em}{\centering HDDC \cite{BGS2006}}
&$20$&$\n8.65$&$1.53$&$\n8.71$&$1.54$&$0.44$&$1.52$\\
&$12$&$\n8.16$&$1.54$&$\n8.14$&$1.53$&$0.42$&$1.52$\\
&$4$&$\n7.95$&$1.54$&$\n7.96$&$1.54$&$0.41$&$1.52$
\end{tabular}
\end{center}
\caption{
Average execution time (in seconds) for the E-step and M-step in the EM algorithm for estimating the parameters of the mixture models.
}
\label{tab_times}
\end{table}

\begin{figure}
\begin{subfigure}[t]{0.25\textwidth}
\centering
\includegraphics[width=\textwidth]{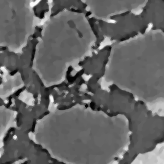}
\end{subfigure}\hfill
\begin{subfigure}[t]{0.25\textwidth}
\centering
\includegraphics[width=\textwidth]{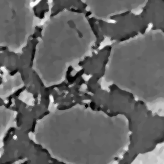}
\end{subfigure}\hfill
\begin{subfigure}[t]{0.25\textwidth}
\centering
\includegraphics[width=\textwidth]{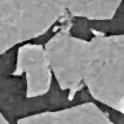}
\end{subfigure}\hfill
\begin{subfigure}[t]{0.25\textwidth}
\centering
\includegraphics[width=\textwidth]{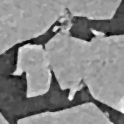}
\end{subfigure}\hfill
\begin{subfigure}[t]{0.25\textwidth}
\centering
\includegraphics[width=\textwidth]{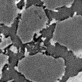}
\end{subfigure}\hfill
\begin{subfigure}[t]{0.25\textwidth}
\centering
\includegraphics[width=\textwidth]{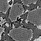}
\end{subfigure}\hfill
\begin{subfigure}[t]{0.25\textwidth}
\centering
\includegraphics[width=\textwidth]{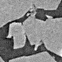}
\end{subfigure}\hfill
\begin{subfigure}[t]{0.25\textwidth}
\centering
\includegraphics[width=\textwidth]{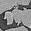}
\end{subfigure}\hfill
\begin{subfigure}[t]{0.25\textwidth}
\centering
\includegraphics[width=\textwidth]{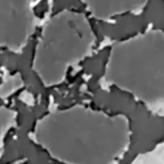}
\end{subfigure}\hfill
\begin{subfigure}[t]{0.25\textwidth}
\centering
\includegraphics[width=\textwidth]{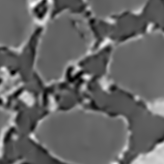}
\end{subfigure}\hfill
\begin{subfigure}[t]{0.25\textwidth}
\centering
\includegraphics[width=\textwidth]{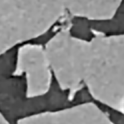}
\end{subfigure}\hfill
\begin{subfigure}[t]{0.25\textwidth}
\centering
\includegraphics[width=\textwidth]{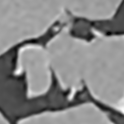}
\end{subfigure}\hfill
\begin{subfigure}[t]{0.25\textwidth}
\centering
\includegraphics[width=\textwidth]{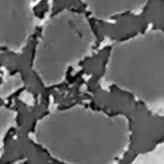}
\end{subfigure}\hfill
\begin{subfigure}[t]{0.25\textwidth}
\centering
\includegraphics[width=\textwidth]{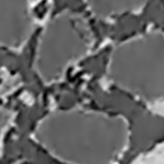}
\end{subfigure}\hfill
\begin{subfigure}[t]{0.25\textwidth}
\centering
\includegraphics[width=\textwidth]{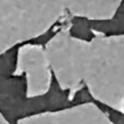}
\end{subfigure}\hfill
\begin{subfigure}[t]{0.25\textwidth}
\centering
\includegraphics[width=\textwidth]{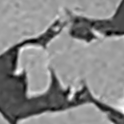}
\end{subfigure}\hfill
\begin{subfigure}[t]{0.25\textwidth}
\centering
\includegraphics[width=\textwidth]{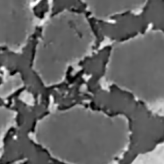}
\caption{Diamonds, $q=2$}
\end{subfigure}\hfill
\begin{subfigure}[t]{0.25\textwidth}
\centering
\includegraphics[width=\textwidth]{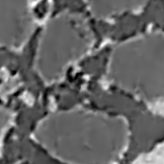}
\caption{Diamonds, $q=4$}
\end{subfigure}\hfill
\begin{subfigure}[t]{0.25\textwidth}
\centering
\includegraphics[width=\textwidth]{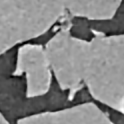}
\caption{FS, $q=2$}
\end{subfigure}\hfill
\begin{subfigure}[t]{0.25\textwidth}
\centering
\includegraphics[width=\textwidth]{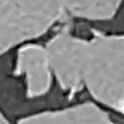}
\caption{FS, $q=4$}
\end{subfigure}\hfill
\caption{Reconstructions of 2D low resolution images. First row: ground truth, second row: low resolution, third row: reconstruction with GMM, fourth row: reconstruction with PCA-GMM and $d=20$, fifth row: reconstruction with PCA-GMM and $d=12$. The larger of $d$, the closer is the result of PCA-GMM to GMM.}
\label{fig_zoom}
\end{figure}

\begin{figure}
\begin{subfigure}[t]{0.5\textwidth}
\centering
\includegraphics[width=\textwidth]{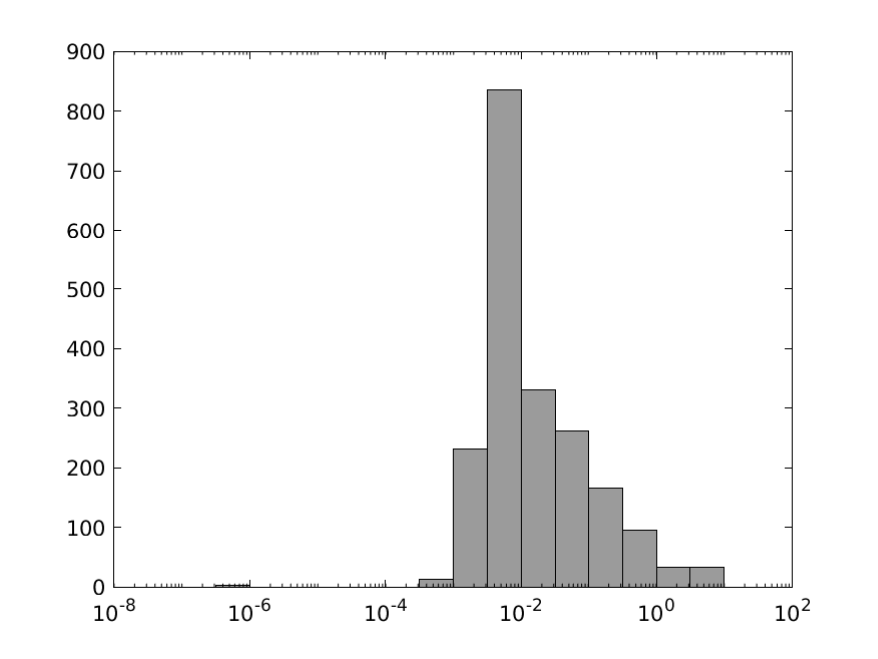}
\subcaption{FS with magnification $q=4$}
\end{subfigure}\hfill
\begin{subfigure}[t]{0.5\textwidth}
\centering
\includegraphics[width=\textwidth]{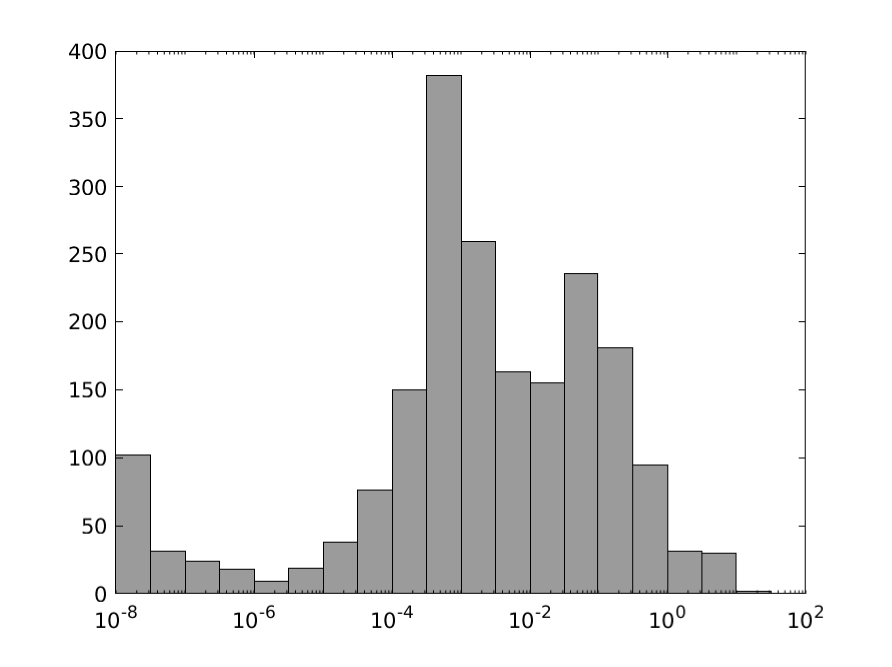}
\subcaption{SiC Diamonds with magnification $q=4$}
\end{subfigure}
\caption{Histograms of the eigenvalues of $\Sigma_k$, $k=1,...,K$ for the PCA-GMM with fixed $\sigma=0.02$ for $d=20$.}
\label{fig_spectra}
\end{figure}
Figure \ref{fig_spectra} shows a histogram of the eigenvalues of the covariance matrices $\Sigma_k$, $k=1,...,K$ of the PCA-GMM model with fixed $\sigma=0.02$ for the FS and SiC Diamonds image with magnification $q=4$. 
We observe, that for the SiC Diamonds image a significant amount of eigenvalues are smaller than $\sigma^2=4\cdot 10^{-4}$ which is not possible within a HDDC model \cite{BGS2006}.
For the FS image, the eigenvalues are mostly greater than $\sigma^2$.

\paragraph{3D-Data.}

In the following, we  present the same  experiments as in the 2D-case but with 3D-data. 
For this we crop a $600\times600\times600$ image from the material images "FS" and "SiC Diamonds". 
For the estimation of the mixture model, we use the upper front left $300\times300\times300$ part of the images and crop randomly $N=1000000$ patches.

Again, we estimate the parameters of a GMM and a PCA-GMM with $K=100$ classes and fixed $\sigma=0.02$ as described in the previous sections.
 Since we have seen in the 2D examples that the results of PCA-GMMs with learned $\sigma$ and HDDC are similar,
we compare our 3D results just with HDDC.
As magnification factor, we use $q=2$. For the low resolution image we use $\tau\times\tau\times\tau$-patches with patch size $\tau=4$ and for the high resolution image we use a patch size of $q\tau=8$. 
We reduce the dimension of the pairs of high and low resolution patches from 
$n=(q^3+1)\tau^3=576$ to $d$ for $d\in\{20,40,60\}$.
After estimating the mixture models, we use the reconstruction method from \cite{SJ2016} 
as described in the previous paragraph to reconstruct the ground truth from of the artificially downsampled images.
The resulting PSNRs are given in Table~\ref{tab_results_3d}
and the average execution times of one E-step and one M-step are given in Table~\ref{tab_times_3d}.
As a reference we also measure the PSNR of the nearest neighbor interpolation.

\begin{table}[htp]
\begin{center}
\begin{tabular}{c|c| c c}
&$d$&FS&Diamonds\\
\hline
Nearest neighbor&-&$30.10$&$26.25$\\
GMM&-&$33.32$&$30.71$\\
\hline
\multirow{3}{6em}{\centering PCA-GMM, $\sigma=0.02$}
&$60$&$33.38$&$30.83$\\
&$40$&$33.36$&$30.75$\\
&$20$&$33.25$&$30.17$\\\hline
\multirow{3}{6em}{\centering HDDC \cite{BGS2006}}
&$60$&$33.23$&$30.49$\\
&$40$&$33.24$&$30.29$\\
&$20$&$33.02$&$29.47$\\
\end{tabular}
\end{center}
\caption{\label{tab_results_3d}
PSNRs of the reconstructions of artificially downsampled 3D images using either nearest neighbor interpolation, 
GMM or PCA-GMM for different choices of $d$. The magnification factor is set to $q=2$.
As in the 2D case,  PCA-GMM with small $d$ produces results almost as good as GMM, but with a much lower dimensionality.
}
\end{table}

\begin{table}[htp]
\begin{center}
\begin{tabular}{c|c| c c | c c}
&&\multicolumn{2}{c|}{FS}&\multicolumn{2}{c}{Diamonds}\\
&$d$&E-step&M-step&E-step&M-step\\
\hline
GMM&-&$717.91$&$\n0.07$&$718.13$&$\n0.07$\\
\hline
\multirow{3}{6em}{\centering PCA-GMM, $\sigma=0.02$}
&$60$&$338.22$&$12.29$&$337.44$&$17.49$\\
&$40$&$327.34$&$\n9.73$&$324.93$&$13.87$\\
&$20$&$320.00$&$\n7.85$&$319.46$&$\n9.80$\\\hline
\multirow{3}{6em}{\centering HDDC \cite{BGS2006}}
&$60$&$337.29$&$\n4.15$&$337.42$&$\n4.16$\\
&$40$&$327.11$&$\n4.19$&$324.95$&$\n4.15$\\
&$20$&$320.03$&$\n4.20$&$319.07$&$\n4.15$\\
\end{tabular}
\end{center}
\caption{\label{tab_times_3d}
Average execution time (in seconds) of the E-step and M-step in the EM algorithm for estimating the parameters of the mixture models.
}
\end{table}

\section{Conclusions} \label{sec:conclusions}
In this paper, we presented a new algorithm to perform image superresolution.
Based on previous work by Sandeep and Jacob \cite{SJ2016}, 
we added a dimension reduction step within the GMM model using  PCA on patches. 
The new variational model, called PCA-GMM is of interest on its own, and can be also applied
for other tasks. 
We solved our PCA-GMM model by an EM algorithm with the usual decreasing guarantees for the objective 
if the E-step and M-step can be performed exactly, see Corollary \ref{cor:EM}.
However, our M-step requires to solve a non-convex constrained minimization problem.
Here we propose a PALM algorithm and prove that all assumptions for the
convergence of the sequence of iterates to a critical point required by \cite{BST2014} are fulfilled,
see Corollary \ref{cor:EM-M}.
Our algorithm has the advantage that the M-step is cheap in relation to  the E-step since it does not rely on the
large numbers of samples in the inner iterations. 

We have demonstrated the efficiency of the new model by numerical examples, 
in the case of 2D and 3D images.
They confirm that PCA-GMM is an efficient way of reducing the dimension of the patches, while keeping almost the same quality of the results than with a GMM algorithm.
This dimension reduction is of the utmost importance when dealing with 3D images, 
where the size of the data gets very large.

As future work, apart from the mathematical analysis of the EM algorithm with approximate M-step, we intend to work on the robustness of the method. This could be done by using a 
robust PCA \cite{NNSS2020}, and also by making the model invariant to contrast changes, see, e.g. \cite{FLS2017}.
Further, we aim to deal with material examples, where we do not subsample the images in a synthetic way.
In particular, we will not know the subsampling operator.
Within ITN MUMMERING such measurements were taken, but require to undergo an advanced registration process.

Finally, we are aware of deep learning techniques for superresolution, see, e.g. \cite{KSKHP2017,iNN_Seg2020}.
We will consider such approaches in the future which would also benefit from 
dimensionality reduction, in particular in 3D.

\section*{Acknowledgment}
 Funding by the German Research Foundation (DFG) with\-in the project STE 571/16-1
 as well as by the French Agence Nationale de la Recherche (ANR) under reference ANR-18-CE92-0050 SUPREMATIM,
 is gratefully acknowledged.
 The EU Horizon 2020 Marie Sklodowska-Curie Actions Innovative Training Network MUMMERIN (MUltiscale, Multimodal and Multidimensional imaging for EngineeRING, Grant Number 765604) is also acknowledged.

\bibliographystyle{abbrv}
\bibliography{ref}
\end{document}